\documentclass{article}

  \usepackage[dblblindworkshop, final]{tagds_2025_corrected}

 \workshoptitle{Topology, Algebra, and Geometry in Data Science}

\usepackage[utf8]{inputenc} 
\usepackage[T1]{fontenc}    
\usepackage{hyperref}       
\usepackage{url}            
\usepackage{booktabs}       
\usepackage{amsfonts}       
\usepackage{nicefrac}       
\usepackage{microtype}      
\usepackage{xcolor}         

\usepackage{multicol}
\usepackage{multirow}
\usepackage{algorithm}
\usepackage{algorithmic}
\usepackage{booktabs} 
\usepackage{wrapfig,xcolor}
\usepackage{natbib}
\usepackage{subfigure}
\usepackage{epsfig,amsmath,amssymb,latexsym,color}

\newtheorem{theorem}{Theorem}
\newtheorem{lemma}{Lemma}
\newtheorem{proposition}{Proposition}

\newtheorem{remark}{Remark}
\newtheorem{assumption}{Assumption}
\newtheorem{definition}{Definition}
\newenvironment{proof}{\begin{trivlist}\item[]{\emph{Proof.}}}
               {\hfill$\Box$\end{trivlist}}

\DeclareMathOperator*{\argmin}{arg\,min}

\title{Advancing Local Clustering on Graphs via Compressive Sensing: Semi-supervised and Unsupervised Methods}

\author{%
  Zhaiming Shen \\
  School of Mathematics\\
  Georgia Institute of Technology\\
  Atlanta, GA 30332 \\
  \texttt{zshen49@gatech.edu} \\
  \And
  Sung Ha Kang \\
  School of Mathematics\\
  Georgia Institute of Technology \\
  Atlanta, GA 30332\\
\texttt{kang@math.gatech.edu} \\
}

\begin{document}

\maketitle

\begin{abstract}
 Local clustering aims to identify specific substructures within a large graph without any additional structural information of the graph. 
These substructures are typically small compared to the overall graph, enabling the problem to be approached by finding a sparse solution to a linear system associated with the graph Laplacian.
In this work, we first propose a method for identifying specific local clusters when very few labeled data are given, which we term semi-supervised local clustering. We then extend this approach to the unsupervised setting when no prior information on labels is available. The proposed methods involve randomly sampling the graph, applying diffusion through local cluster extraction, then examining the overlap among the results to find each cluster. We establish the co-membership conditions for any pair of nodes, and rigorously prove the correctness of our methods. Additionally, we conduct extensive experiments to demonstrate that the proposed methods achieve state of the art results in the low-label rates regime.
\end{abstract}

\section{Introduction}
The ability to learn from data by uncovering its underlying patterns and grouping it into distinct clusters based on latent similarities and differences is a central focus in machine learning.
Over the past few decades, traditional clustering problems have been extensively studied as an unsupervised learning task, leading to the development of a wide range of foundational algorithms, such as $k$-means clustering~\citep{MacQueen67}, density-based clustering~\citep{Ester96}, spectral clustering~\citep{Ng01,ZelnikManor04}, hierarchical clustering~\citep{Nielsen16} and regularized $k$-means \citep{kang2011regularized}. These foundational methods have, in turn, inspired numerous variants and adaptations tailored to specific data characteristics or application domains.

Researchers have also developed semi-supervised learning approach for clustering, which leverages both labeled and unlabeled data in various learning tasks. One of the most commonly used methods in graph-based semi-supervised learning is Laplace learning \citep{ZZL03}. Note that Laplacian learning sometimes is also called Label propagation \citep{ZG02}, which seeks a graph harmonic function that extends the labels. Laplacian learning and its variants have been extensively applied in semi-supervised learning tasks \citep{Zhou04a,Zhou04b,Zhou05,Ando07,kang2014supervised}. A key challenge with Laplacian learning type of algorithms is their poor performance in the low-label rates regime. To address this, recent approaches have explored $p$-Laplacian learning \citep{Alaoui16,Flores22}, higher order Laplacian regularization \citep{Zhou11}, weighted nonlocal Laplacians \citep{Shi17}, properly weighted Laplacian~\citep{Calder19}, and Poisson learning \citep{Calder20}.

Those aforementioned clustering algorithms are all global clustering algorithms, recovering all cluster structures simultaneously. However,  real-world applications often require identifying only specific substructures within large, complex networks. For example, in a social network, an individual may only be interested in connecting with others who share similar interests, while disregarding the rest of individuals. In such cases, global clustering methods become inefficient, as they generate excessive redundant information rather than focusing on the relevant local patterns.

In this paper, we turn our focus to a more flexible clustering method called \emph{local clustering} or \emph{local cluster extraction}. 
Local clustering focus only on finding one target cluster which contains those nodes of interest to us, and disregard the non-interest or background nodes. Researchers have investigated in this direction over the recent decades such as \citet{Lang04,Andersen06,Kloster14,SpielmanTeng2013,Andersen2016,Veldt19,Fountoulakis20}.  More recently, inspired by the idea of compressive sensing, \citet{LM20,LS23,SLL23} proposed a novel perspective for local clustering by finding the target cluster via solving a sparse solution to a linear system associated with the graph Laplacian matrix. 
\citet{shen2024sparse} provided a comprehensive summary of the compressive sensing based local clustering approaches. Such approaches can also be applied in medical imaging as demonstrated in \citet{hamel2024local}.
Intuitively speaking, as local clustering only focus on finding ``one cluster at a time", it is flexible in practice and can be more efficient and effective than global clustering algorithms.

Besides such merits that local clustering algorithms possesses, one big downside of it is the necessity of having the initial nodes (we call them seeds) of interest as prior knowledge, and it also requires a good estimate of the target cluster size. The seeds information is sometimes very limited and even unavailable, which makes local clustering algorithm less popular.  
We propose a clustering method which requires 
very few seeds (semi-supervised case) or no seeds (unsupervised case), see Figure~\ref{SSLC_explained} and \ref{USLC_explained} for illustration. 
We provide a detailed discussion of the procedure, with analysis on the correctness of the proposed method.  The main contributions of our work are as follows:

\begin{enumerate}
    \item We propose a semi-supervised (with very few labeled data) and an unsupervised (no labeled data) local clustering methods which outperform the state-of-the-arts local and non-local clustering methods in most cases.

    \item For theoretical considerations, we relax the assumption on the spectral norm of perturbation on the graph Laplacian and prove the correctness of our method in the semi-supervised case. We also establish the co-membership conditions for any pair of nodes, and then prove the correctness of our method in the unsupervised case. 

    \item Computationally the proposed method can also show benefits that our semi-supervised method can find all the clusters simultaneously which improves the ``one cluster at a time" feature of local clustering algorithms in terms of efficiency. We provide extensive experiments with comparisons to show the effectiveness of our methods in the low-label rates regime.
\end{enumerate}

The rest of the paper is organized as follows. 
In Section~\ref{secPrelim}, we review the necessary background and state the model assumptions for our tasks. In Section~\ref{sec:SSLC} and \ref{sec:USLC},
we present step-by-step procedure of the proposed methods in both semi-supervised and unsupervised settings respectively, and prove their correctness under certain assumptions. In section~\ref{secExp}, we show the experimental results of the proposed methods and compare them with the state-of-the-art semi-supervised clustering algorithms on various benchmark datasets. Finally, in Section~\ref{secConclusion}, we conclude the paper, discuss its limitations and societal impact.

\section{Preliminaries and Model Assumptions}
\label{secPrelim}

For a graph $G=(V, E)$, we use $V$ to denote the set of all nodes, and $E$ to denote the set of all edges. Suppose $G$ has $k$ non-overlapping underlying clusters $C_1, C_2, \cdots, C_k$, we use $n_i$ to denote the size of $C_i$ where $i=1,2,\cdots,k$, and use $n$ to denote the total size of graph $G$.
We use $A$ to denote the adjacency matrix (possibly weighted but non-negative) of graph $G$, and use $D$ to denote the diagonal matrix whose diagonal entries is the degree of the corresponding vertex. {In this paper, we focus on graphs without node features and define a cluster in terms of edge connectivity: a cluster is a subset of nodes that is densely connected internally and sparsely connected to the rest of the graph. We consider two settings for local clustering: semi-supervised and unsupervised. For semi-supervised case, we refer to seed(s) as a given initial set of nodes with known labels. For unsupervised case,  there is no seed(s) available.}

{For theoretical analysis purpose, we make the following assumptions on the graph model.




\begin{assumption} \label{assump1}
    The non-zero entries in the diagonal block is denser than the non-zero entries in the off-diagonal block (after permutation according to the cluster membership). More precisely, we assume the graph Laplacian satisfies $\|\Delta L\|_2=\|L-L^{in}\|_2 = o(1)$ and {$(\delta_n(L))^{\log n}=o(1)$} as the graph size $n\to\infty$. The definition of RIP constant $\delta_n$ is provided in Definition \ref{def:RIP} in Appendix \ref{secAppendixCS}.
\end{assumption}

\begin{assumption} \label{assump2}
The number of cluster $k=O(1)\geq 2$ and there are no overlapping between clusters.
    The size of each cluster is not too small nor too large, i.e., there exist $n_{\min}=\Theta(n)$ and $n_{\max}=\Theta(n)$ such that $n_{\min}\leq n_i\leq n_{\max}$ for any $i\geq 1$.
\end{assumption}

\begin{assumption} \label{assump3}
For any distinct $i,j\in C_s$ with some $s\geq 1$, let $K_{i,j}\subset C_s$ be the set of nodes such that given any node in $K_{i,j}$ as the seed, LCE always finds node $i$ and $j$ into the same local cluster. We assume the cardinality of $K_{i,j}$ is at least $(1-o(1))\frac{n^2_{\min}}{n}$ for any distinct pair of nodes $i,j$. 
\end{assumption}

Throughout the paper, all the notation $o(1)$ is with respect to the size of graph $n\to \infty$.

Assumption \ref{assump1} guarantees that the graph exhibits a block structure, ensuring that any density-based clustering algorithm has the potential to succeed as in 
\citet{LM20,LS23,SLL23}. Assumption \ref{assump2} ensures that there are no dominating or dominated clusters, yet allows for outliers, 
for example, nodes lying outside all clusters are permitted. 
Assumption \ref{assump3} is essentially a type of homogeneity assumption on each cluster of the graph in order to guarantee the performance of LCE (Algorithm \ref{alg1}) on a randomly chosen node. 

In this work, we pursue the idea of using compressive sensing for local clustering tasks. More details of the connection between compressive sensing and local clustering are provided in Appendix \ref{secAppendixLCE} and \ref{secAppendixCS}.  One notable theoretical contribution in this work is that we are able to relax the assumption from $\|\Delta L\|_2=o(n^{-1/2})$ (see Lemma \ref{IndAnaLCE} in Appendix \ref{secAppendixLCE}) \citep{SLL23} to $\|\Delta L\|_2=o(1)$  by choosing the parameters more carefully and enforcing tighter bounds on the inequalities (see Lemma \ref{IndAna} in Appendix \ref{secAppendixOtherLemmas}). 



\section{Semi-supervised Local Clustering} \label{sec:SSLC}

For the semi-supervised case, we consider a graph $G=(V,E)$ with non-overlapping underlying cluster $C_1, \cdots, C_k$. We are interested in the following two local clustering setups:
\begin{enumerate}
    \item Suppose a very small portion of seeds $\Gamma_1 \subset C_1$ is available, we would like to find all the nodes in the target cluster $C_1$.
    \item Suppose a very small portion of seeds $\Gamma_i\subset C_i$ is available for each $i=1,\cdots,k$, we would like to assign all nodes to their corresponding underlying clusters.
\end{enumerate}

Similar to the issue of Laplacian learning type of algorithms, local clustering approach such as LCE (see Algorithm \ref{alg1}) becomes less effective in the low-label rates regime. Therefore one wants to extract more seeds from each cluster before applying LCE. 
We illustrate the idea in Figure~\ref{SSLC_explained}. We assume our target cluster is $C_1$, which is the cluster in the left of those three clusters illustrated in the first subplot of Figure~\ref{SSLC_explained}. 

\begin{figure}
    \centering
\includegraphics[width=0.75\linewidth]{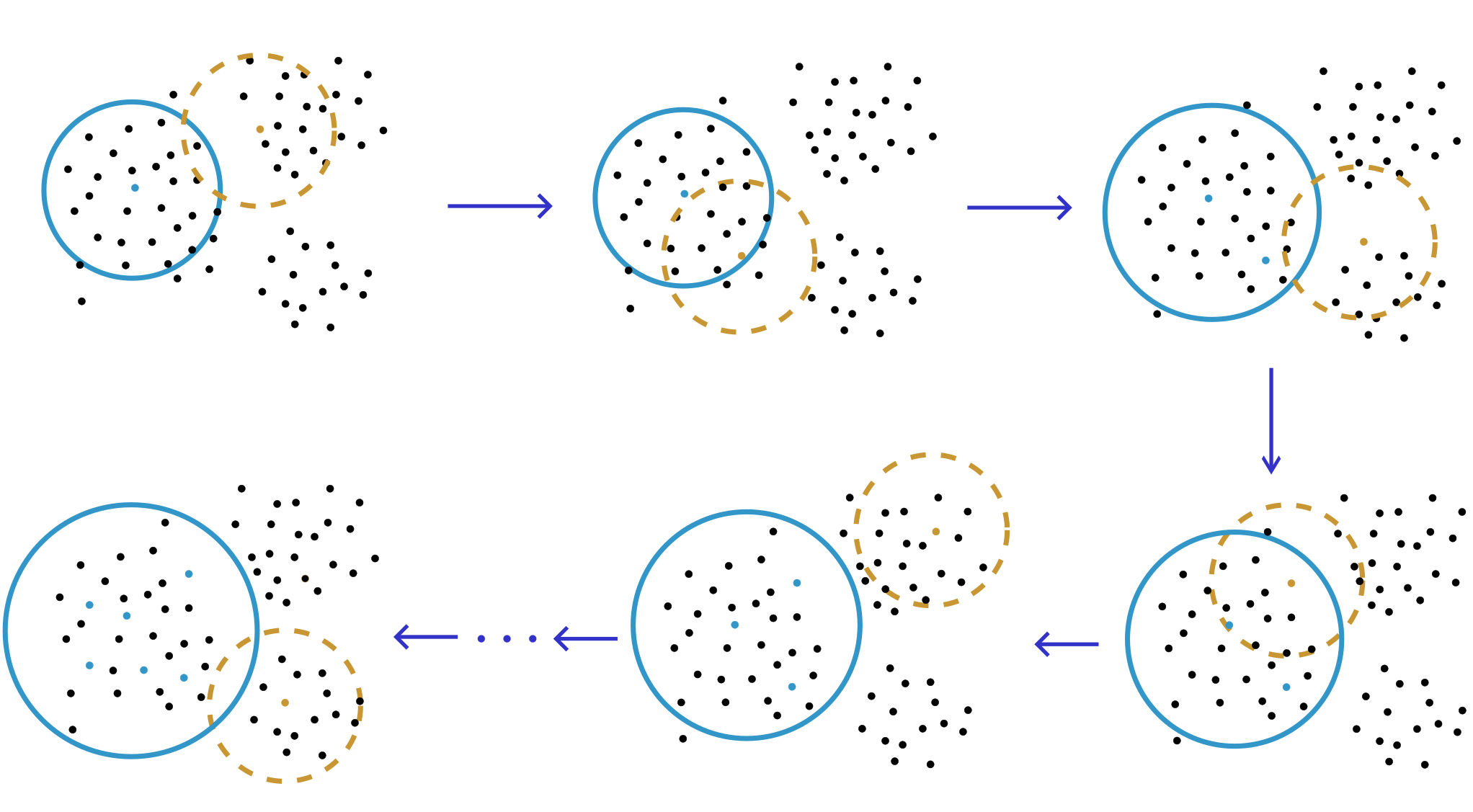}
\caption{Illustration of semi-supervised local clustering (SSLC) for a single cluster. Each subplot indicates one iteration (Blue dots: seeds in $\Gamma_1$. Brown dots: randomly sampled node in each iteration).}
\label{SSLC_explained}
\end{figure}

The procedure of semi-supervised local clustering (SSLC) is summarized as follows:

\begin{enumerate}
    \item Given a graph $G$, and initial seed(s) $\Gamma_1$ (indicated as the blue point in the first subplot), we first apply LCE to have a rough estimate $\tilde{C}_1$ (indicated by the solid blue circle) of $C_1$. 
    \item Randomly sample a node $v$ (the brown point in the first subplot) in $G$ and apply LCE to get a rough cluster around $v$ (indicated by the dashed brown circle). Then check the overlap between solid blue and dashed brown circle. If the overlap contains the majority of the nodes in the brown circle, then add $v$ into $\Gamma_1$, otherwise sample another node.
    \item Continue this process and keep increasing the size of $\Gamma_1$ until a predetermined iteration number is reached (indicated in the process of from the second to the last subplots). 
\end{enumerate}

Note that this procedure can be applied when more than one cluster is of interest without increasing the number of iterations. The procedures for the single cluster case with $\Gamma_1$ and $C_1$ and multiple clusters case $\Gamma_s$ and $C_s$ are summarized in Algorithm~\ref{alg3} and \ref{alg6} respectively. This method finds 
more nodes which can be added into the initial set of seeds, and the user can flexibly determine the number of seeds wanted by increasing or decreasing the number of iterations. Theorem \ref{thmSSLC} essentially establishes the correctness of Algorithm~\ref{alg3} and \ref{alg6}.

\begin{algorithm}
\caption{Semi-supervised Local Clustering (SSLC) for a Single Cluster \label{alg3}}
\begin{algorithmic}[1]
\REQUIRE The adjacency matrix $A$ of an underlying graph $G$, the initial seed(s) $\Gamma_1$ for the target cluster $C_1$, the estimated size ${n}_1$, the number of resampling iteration $\ell$
\ENSURE  desired output cluster $C_1^{\#}$
\STATE{$\tilde{C}_1 \leftarrow{\rm LCE}(A,{n}_1,\Gamma_1)$}
\FOR{$i=1:\ell$}
\STATE{$v\leftarrow$ uniformly random sampled node from $G$}
\STATE{$C^{\#} \leftarrow{\rm LCE}(A,{n}_1,v)$}
\IF{$|\tilde{C}_1\cap C^{\#}| > 0.5|C^{\#}|$}
\STATE{$\Gamma_1\leftarrow \Gamma_1\cup\{v\}$}
\STATE{$\tilde{C}_1 \leftarrow{\rm LCE}(A,{n}_1,\Gamma_1)$}
\ENDIF
\ENDFOR
\STATE{$C_1^{\#} \leftarrow \tilde{C}_1$ }
\end{algorithmic}
\end{algorithm}

\begin{theorem} \label{thmSSLC}
Suppose $G$ satisfies Assumptions~\ref{assump1} - \ref{assump3}. Then when $n$ (the size of $G$) gets large, for each $s=1,\cdots,k$, we have
\[ \begin{cases} |\tilde{C}_s\cap C^\#| 
    > (1-o(1))|C^\#|, & \text{if}\text{ } v\in C_s, \\
    |\tilde{C}_s\cap C^\#|
      \leq o(|C^\#|), & \text{otherwise.} \\
   \end{cases}
\]
\end{theorem}

Theorem \ref{thmSSLC} 
    shows that whenever a node $v$ being added to $\Gamma_1 (\Gamma_s)$ based on the large overlap criterion, it satisfies $v\in C_1 (C_s)$. This means that $\Gamma_1 (\Gamma_s)$ only grow with the nodes within $C_1 (C_s)$, which is the essential step to guarantee the correctness of Algorithm \ref{alg3} and Algorithm \ref{alg6}.

%


\section{Unsupervised Local Clustering} \label{sec:USLC}

For unsupervised case, we consider a graph $G=(V,E)$ with non-overlapping underlying cluster $C_1, C_2, \cdots$.
Assume neither the prior information about seeds nor the number of clusters is given, the goal is to assign all the nodes to their corresponding underlying clusters.

In this case, we randomly sample and build a local cluster from the sampled node every time.  We check its overlap with the local cluster obtained from any newly sampled node.  After a certain number of iterations, the nodes from the same underlying cluster are more likely to be clustered together. We  illustrate the idea of unsupervised local clustering in Figure~\ref{USLC_explained}. The procedure of unsupervised local clustering (USLC) is summarized as follows:

\begin{enumerate}
    \item Given a graph $G$, in each iteration, we randomly sample a node and find its local cluster via LCE (as shown in the top row of Figure~\ref{USLC_explained}). 
    
    \item 
    Based on the found local cluster, build the co-membership matrix in such a way that it outputs $1$ in the location $(i,j)$ when both $i,j$ are contained in that found cluster, and outputs $0$ otherwise.

    \item Aggregate the clustering results from all iteration into a co-membership matrix $M$ (roughly speaking, each entry $M_{i,j}$ represents the probability of a pair of nodes being clustered into the same local cluster).  The values in each block of matrix $M$ asymptotically converges  to the same value, as the iteration number goes up (as shown in the bottom row of Figure~\ref{USLC_explained}).

    \item Randomly sample a node $i$, and use a prescribed cutoff number $\delta$ to select all the nodes $u$ such that $M_{i,j}>\delta$ and assign those $j$'s as one of the clusters $C^\#$.

    \item Delete the subgraph generated by $C^\#$ from $G$. Then keep iterating until the prescribed maximum number of clusters is reached or the remaining size of graph is too small.
\end{enumerate}

\begin{figure}
    \centering
    \begin{tabular}{c}
\includegraphics[width=0.9\linewidth]{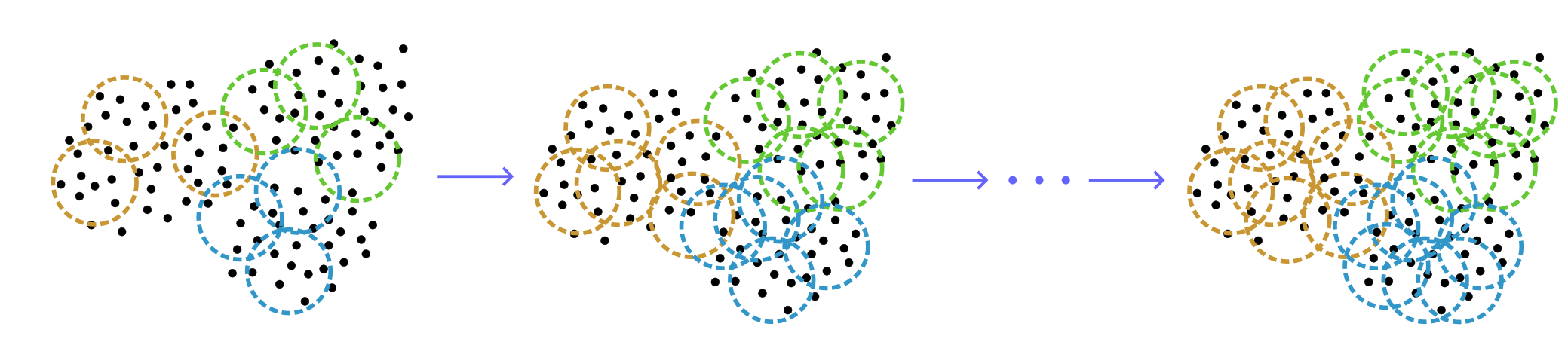} \\
\hspace{-10mm}
\includegraphics[width=0.95\linewidth]{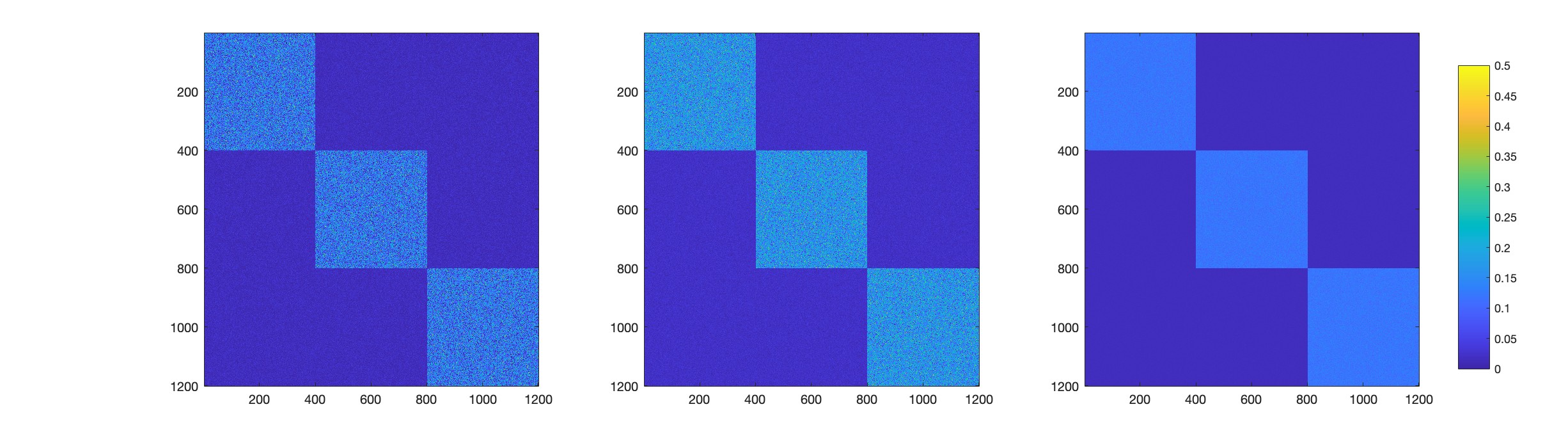} 
    \end{tabular} 
    \vspace{-2mm}
\caption{Illustration of USLC procedure. Top row: each dashed circle represents one iteration of LCE generated from a randomly sampled node, different colors indicate local clusters generated from nodes of different underlying clusters. Bottom row: Aggregated co-membership matrix.}
    \label{USLC_explained}
\end{figure}

We summarize its procedure in Algorithm \ref{alg5}. 
We first show in Proposition \ref{Mmatrix} that $M_{i,j}$, the probability of a pair of nodes coming from the same underlying cluster, is significantly different from the probability of a pair of nodes coming from two different underlying clusters. Then we show in Theorem \ref{thmUSLC} the correctness of Algorithm \ref{alg5}. The detailed proofs are deferred to Appendix \ref{secAppendixOtherLemmas}. In Step $4$ 
above, we take a thresholding $\delta$, since the values is between $0$ and $1$ after taking the average in Step $3$.  This thresholding value is given in the proof of Theorem \ref{thmUSLC} in Appendix \ref{secAppendixOtherLemmas}.

\begin{algorithm}
\caption{Unsupervised Local Clustering (USLC) \label{alg5}}
\begin{algorithmic}[1]
\REQUIRE The adjacency matrix $A$ of an underlying graph $G$, the number of resampling iteration $\ell$
\ENSURE  The desired clusters $C_1^{\#},C_2^{\#},\cdots$
\STATE{initialize $s\leftarrow 1$}
\WHILE{$|G|<n_{\min}$}
\STATE{initialize the comembership matrix $M$ as zero matrix}
\FOR{$i=1:\ell$}
\STATE{$v\leftarrow$ uniformly random sampled node from $G$}
\STATE{$C^{\#} \leftarrow{\rm LCE}(A,n_{\min},v)$}
\STATE{$M\leftarrow M +\frac{1}{\ell}\cdot{\rm comembership}(\mathbf{1}_{C^\#})$ }
\ENDFOR
\STATE{select a node $v$ uniformly random such that $M_{v,u}>\delta$ for some $u\neq v$, with some appropriate $\delta\in (0,1)$, then find all the $j$ such that $C_s^\# := \{j:M_{v,j}>\delta$\}}
\STATE{Let $G^{(s)}$ be the subgraph spanned by $C_s^\#$}
\STATE{$G\leftarrow G\setminus G^{(s)}$}
\STATE{$s\leftarrow s+1$}
\ENDWHILE
\end{algorithmic}
\end{algorithm}


\begin{proposition} \label{Mmatrix}
Suppose $G$ satisfies Assumptions~\ref{assump1} - \ref{assump3}.
Then, as $n$ gets large, the co-membership matrix $M$ obtained from Algorithm~\ref{alg5}  has a clear block diagonal form (after permutation) as $L\to\infty$. More precisely, the entries in $M$ satisfy
\[\begin{cases} 
    M_{i,j} \geq (1-o(1))\frac{n_{\min}}{n}, & \text{for any}\text{ } i=j, \\
    M_{i,j}  \geq (1-o(1))\frac{n_{\min}^2}{n^2}, & \text{for}\text{ } i,j\in C_s, i\neq j, s\geq 1, \\
     M_{i,j}  \leq o\left(\frac{4n^2_{\min}(1-o(1))}{n^2}\right), & \text{otherwise.}    
   \end{cases}
\]

\end{proposition}

\begin{theorem} \label{thmUSLC}
Suppose the assumptions in Theorem~\ref{Mmatrix} hold. For any $v\in C_s, s\geq 1$, there exists some $\delta\in (0,1)$ and $C_s^{\#}=\{i:M_{v,i}>\delta\}$ such that $C_s^\#=C_s$. Consequently, Algorithm~\ref{alg5} assigns all nodes into their clusters correctly.
\end{theorem}

It is worth of noting that Algorithms \ref{alg3}, \ref{alg5} and \ref{alg6} are applicable if there are outlier nodes, i.e., nodes which do not belong to any underlying clusters, presented in the graph. In such case, we extract all the clusters, and the remaining unclustered nodes are automatically classified as outlier nodes. 

\section{Experiments} \label{secExp}

We evaluate the performance of SSLC and USLC on both synthetic and real datasets, with a particular focus on clustering in the low-label rates regime (with label rates at most $1\%$). Additionally, we demonstrate the robustness of our methods by manually introducing outliers into the dataset and assessing their impact. All experiments are conducted on a local machine with 8-core Ryzen 7 7700X CPUs and 24 GB of RAM capacity. The runtime of SSLC/USLC is primarily determined by the LCE procedure, which is $O(nd_{\max}\log(n))$. In the regime $d_{\max}=O(\log n)$, this becomes $O(n\log^2 n)$. Therefore, if there are total $k$ number of clusters, the total run time is $O(kn\log^2 n)$. Samples of demo code is available at \url{https://github.com/zzzzms/Iterative_LCE}.

\begin{figure}
    \centering   
    \begin{tabular}{ccc}  
    \includegraphics[width=0.31\linewidth]{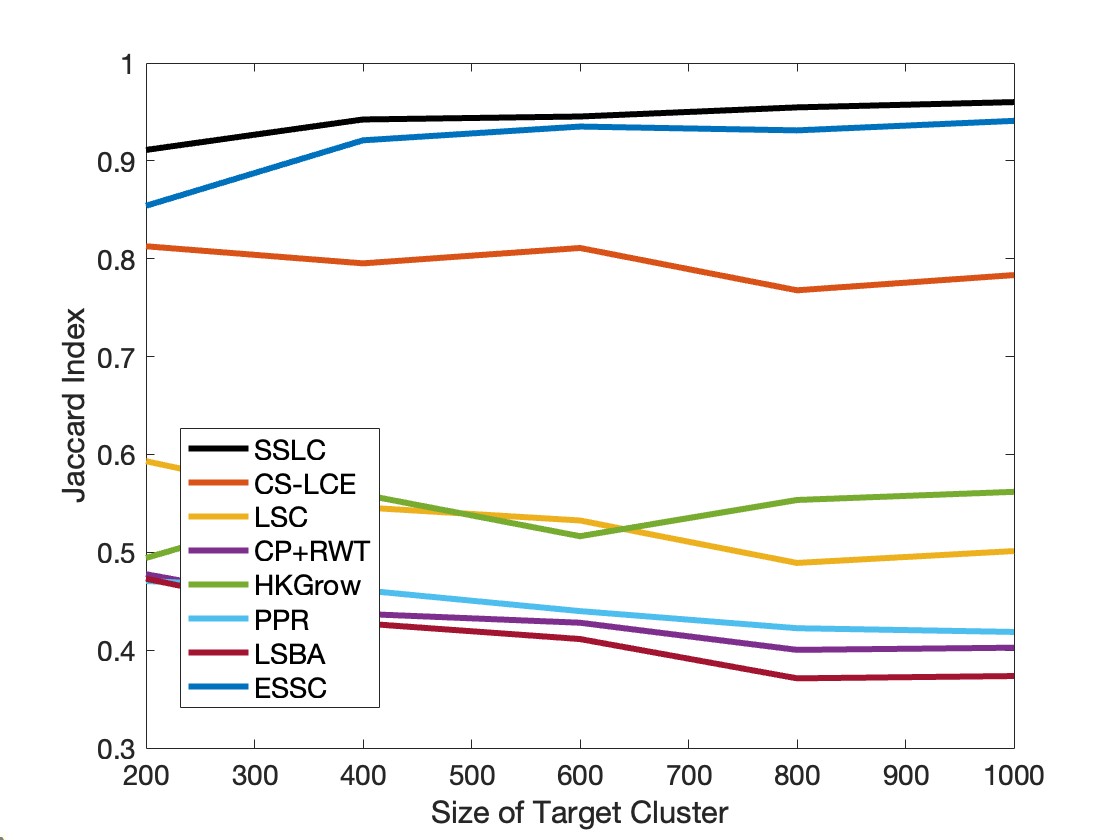} & \includegraphics[width=0.31\linewidth]{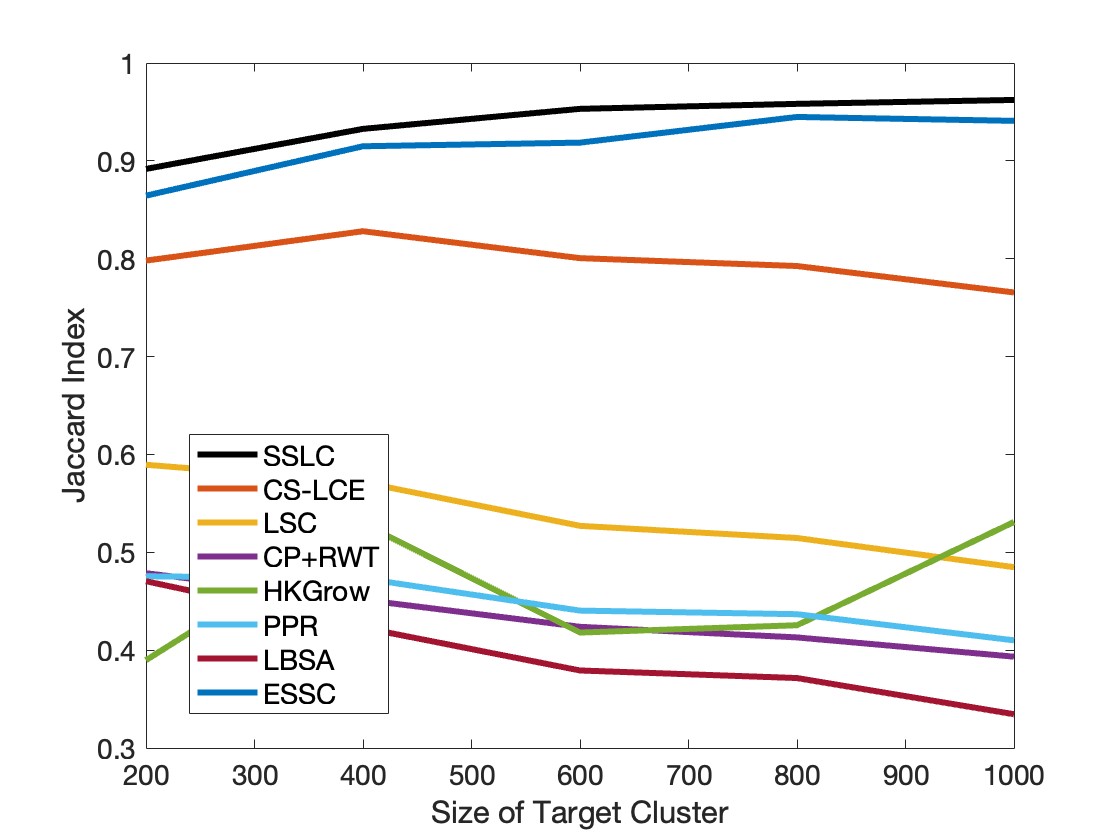} & \includegraphics[width=0.31\linewidth]{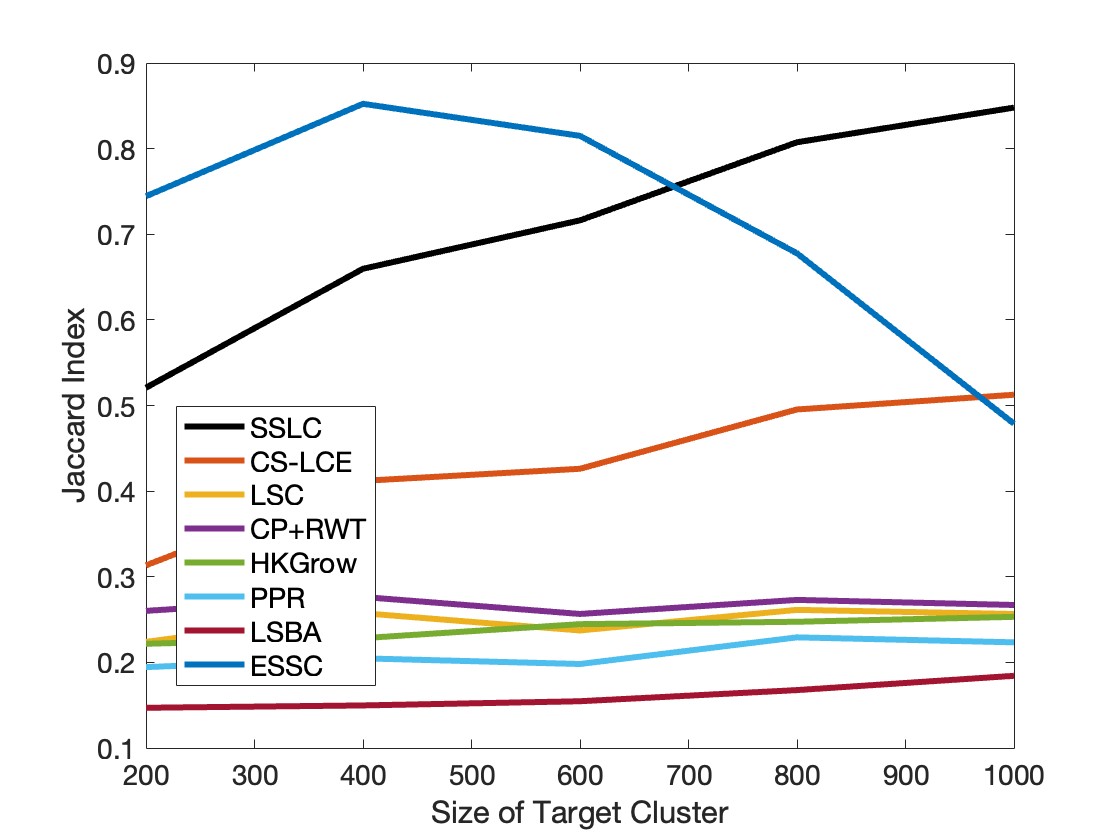} \\
    \includegraphics[width=0.31\linewidth]{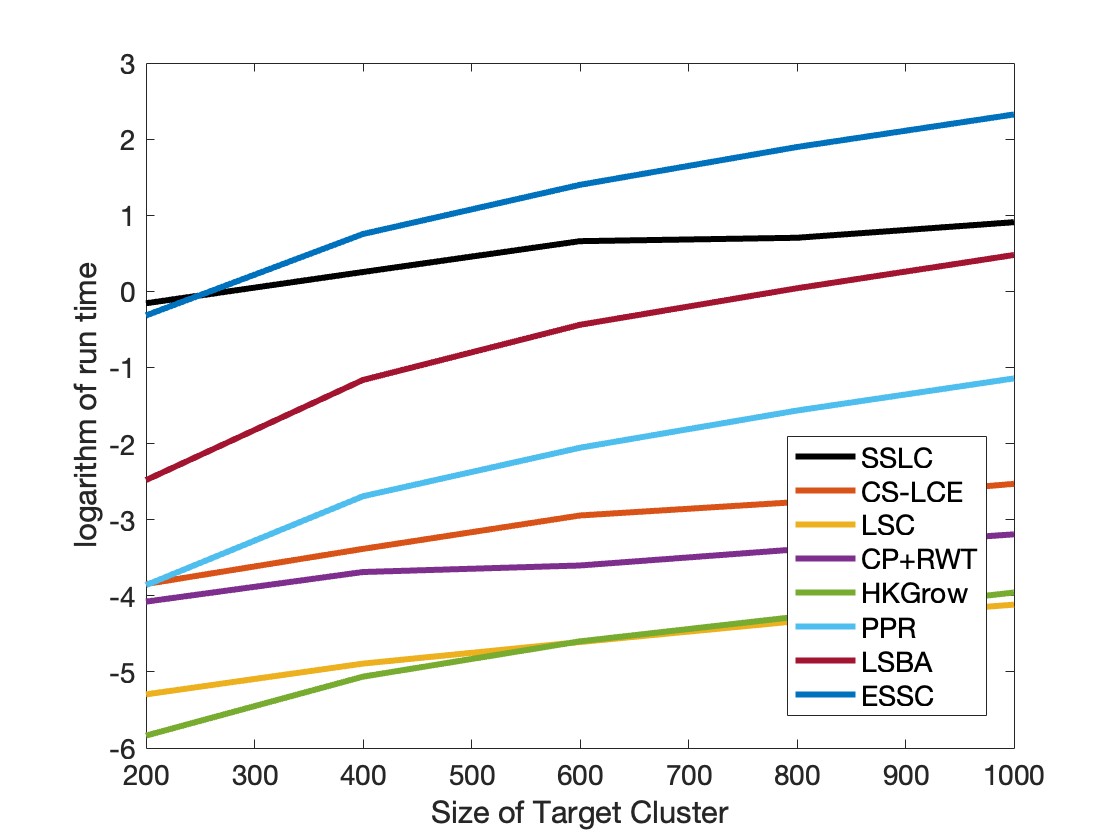}  & \includegraphics[width=0.31\linewidth]{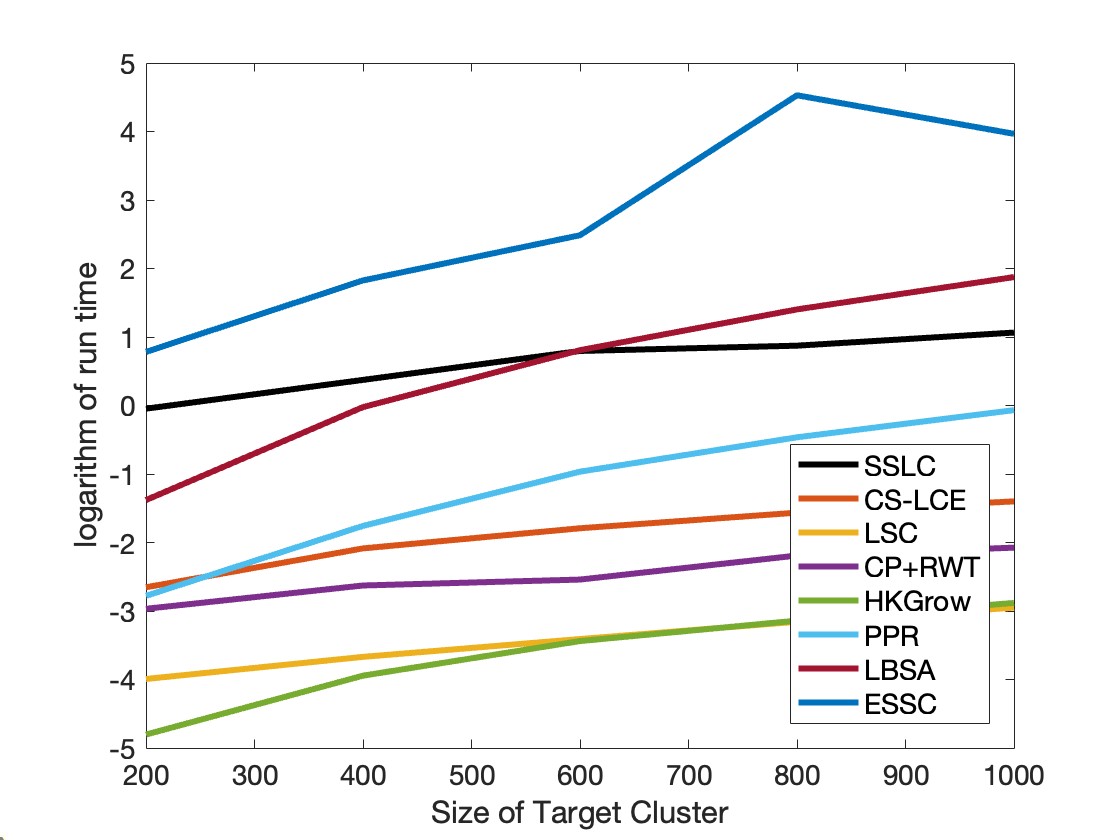} &
 \includegraphics[width=0.31\linewidth]{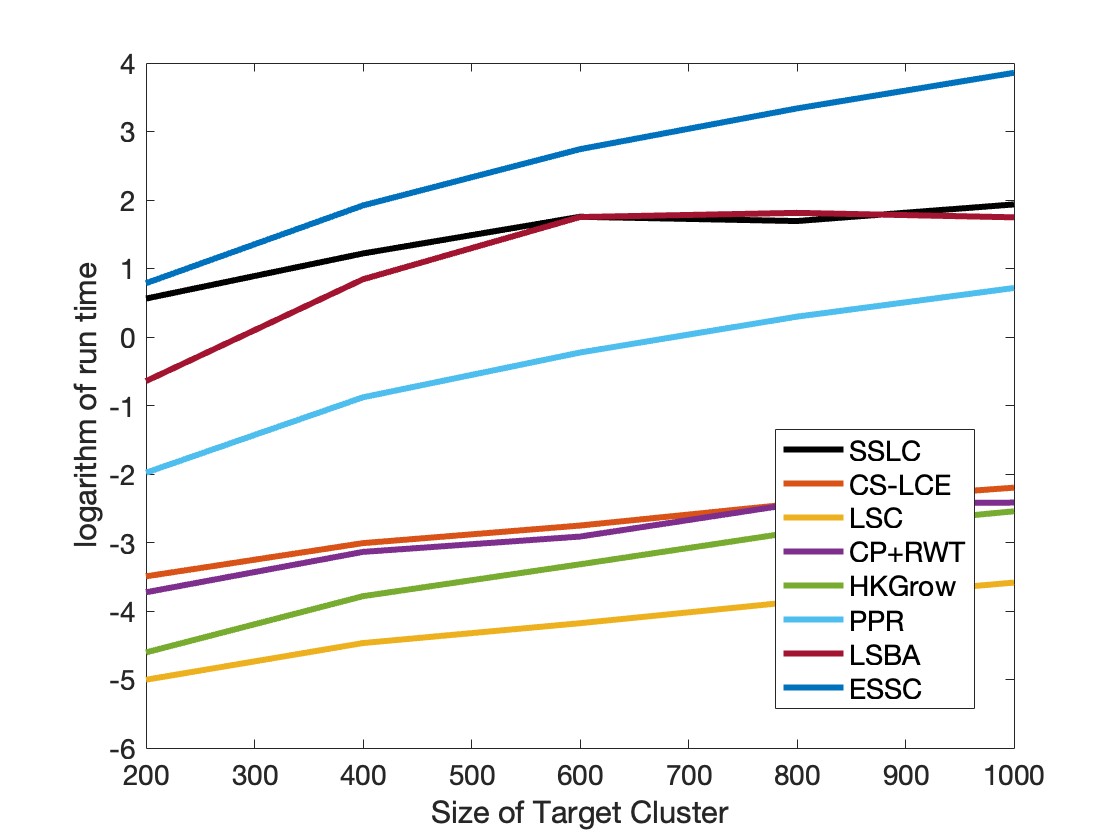} 
    \end{tabular}
    \vspace{-2mm}
\caption{Plots of Jaccard index (over 100 trials) and logarithm of running time of SSLC for stochastic block model. 
(a) The first column shows experiment with three equal cluster size and single cluster $C_1$ extraction.  (b) The middle column shows experiment with three equal cluster size and all clusters $C_s$ extraction. (c) The right column  shows experiment with unequal cluster sizes, focusing on the most dominent cluster. 
}
\label{figSBM}
\end{figure}
\paragraph{Synthetic Data}
We first conduct two experiments for  semi-supervised setting on stochastic block model (SBM) with different cluster sizes and different connection probabilities in Figure~\ref{figSBM}. The first cases (a) and (b) are symmetric SBM with three equal cluster size, the second case (c) is general SBM with three unequal clusters sizes. 
For the symmetric case, we consider both the single cluster extraction in (a) and all clusters extraction in (b).  For the nonsymmetric case in (c), we only focus on extracting the most dominant cluster, i.e., the cluster with the largest connection probability.  The parameters for generating data in both cases are presented in Appendix~\ref{secAppendixImplementation}.
We compare against several other local clustering algorithms such as CS-LCE or LCE~\citep{SLL23}, LSC~\citep{LS23}, CP+RWT~\citep{LM20}, HKGrow~\citep{Kloster14}, PPR~\citep{Andersen06}, ESSC~\citep{Wilson14}, and LBSA~\citep{Shi19}.  

For the experiments on SBM, we set the number of initial seeds to be 1. 
We observe that SSLC has a better Jaccard index compared to other methods in most cases. Moreover, for symmetric model, the Jaccard index in both  single cluster extraction and all clusters extraction are similar, while the running time of all clusters extraction is more advantageous towards SSLC (see also Remark~\ref{rmktime4SSLC}). 

To further validate that our assumption $\|\Delta L\|_2=\|L-L^{in}\|_2=o(1)$, we conduct experiments on symmetric stochastic block model by fixing the number of nodes, number of clusters, intra-connection probability, while varying the inter-connection probability. The results are provided in Table \ref{DeltaLTable}. We observe that the $\|\Delta L\|_2$ decreases as $n$ increases, which is as desired.

    Our experimental results are also connected to the theoretical results in the SBM literature. Note that when the signal-to-noise ratio ${\rm SNR} = \frac{(a-b)^2}{k(a+(k-1)b)}>1$ (Kestem-Stigum threshold) in the SBM literature gives a theoretical bound for weak recovery of SBM in the sparse regime (with connectivity $\frac{a}{n}$ and $\frac{b}{n}$). For our exact recovery case (with connectivity $\frac{a\log n}{n}$ and $\frac{b\log n}{n}$), if we fix $a,b,k$ and increase $n$, then $\|\Delta L\|$ decreases (illustrated in Table \ref{DeltaLTable}). In our experiments, we set $k=3, p=\frac{a\log n}{n}, q=\frac{b\log n}{n}$ with $a=6$ and $b=1$, then we roughly satisfies the condition $\sqrt{a}-\sqrt{b}\approx\sqrt{k}$ for exact recovery, which is given in Theorem 14 and Remark 16 in \citet{abbe2018community}.

Besides SBM, we compare our algorithm against other compressive sensing based local clustering methods on geometric dataset consisting of different shapes (three lines, three circles, three moons). The 2D projection and more details of this dataset are provided in Appendix~\ref{secAppendixGeo}. Note that spectral clustering often fails on these datasets when the noise level is large. However, compressive sensing based approaches work better.
We use $1$ label per class, the clustering results are shown in Table~\ref{GeoTable}. We observe that, on all three datasets, Lines, Circles, and Moons, our method SSLC consistently achieves high accuracy, and it outperforms other compressive sensing based local clustering methods, especially CP+RWT and LSC.

\begin{table}
\centering
\caption{Average accuracy and standard deviation (\%) on geometric data over 100 trials}
\label{GeoTable}
	\begin{tabular}{l|ccc}
	\toprule
           Datasets & 3 Lines & 3 Circles & 3 Moons  \\
    \midrule
    CP+RWT & 82.1 \text{ }(9.1) & 96.1 \text{ }(5.1) & 85.4 \text{ }(1.3)  \\
	LSC & 89.0 \text{ }(5.5) & 96.2 \text{ }(3.7) & 85.3 \text{ }(1.9) \\
	
 LCE & 92.4 \text{ }(8.1) & 97.6 \text{ }(4.7) & 96.8 \text{ }(0.9)  \\
    
    SSLC (ours) & \textbf{94.8\text{ }(7.2)} & \textbf{98.2 \text{ }(4.1)} & \textbf{97.3\text{ }(1.2)}  \\
	\bottomrule
	\end{tabular}
\end{table}

\paragraph{Real Data} 
For real datasets, we focus on clustering in the low-label rates regime on FashionMNIST~\citep{Xiao17}, CIFAR-10~\citep{Krizhevsky09}, and Planetoid (Cora, CiteSeer, PubMed) \citep{Yang2016}. To apply graph-based clustering algorithms on images, we first construct an auxiliary K-NN graph, and compute the pairwise distance from Gaussian kernel based on the Euclidean distance of the latent features with some scaling factors. Similar constructions have also appeared among others \citep{ZelnikManor04,Jacobs18,Calder20}. More detailed construction is provided in Appendix~\ref{secAppendixImplementation}. 

\begin{table}
\caption{Average accuracy and standard deviation (\%) on FashionMNIST over 30 trials} \label{tabFashionMNIST}
\centering
\begin{tabular} {lllllll}
\toprule
 \# Labels per class & 0 & 1  & 2 & 3 & 4 & 5  \\ 
 \midrule
Laplace & - & 18.4 (7.3) & 32.5 (8.2) & 44.0 (8.6) & 52.2 (6.2)  & 57.9 (6.7)  \\
Random Walk & - & 49.0\text{ }(4.4) & 55.6\text{ }(3.8) & 59.4\text{ }(3.0) & 61.6\text{ }(2.5) & 63.4\text{ }(2.5)  \\
VolumeMBO & - & 54.7\text{ }(5.2) & 61.7\text{ }(4.4) & 66.1\text{ }(3.3) & 68.5\text{ }(2.8) & 70.1\text{ }(2.8) \\
WNLL & - & 44.6\text{ }(7.1) & 59.1\text{ }(4.7) & 64.7\text{ }(3.5) & 67.4\text{ }(3.3) & 70.0\text{ }(2.8)  \\
$p$-Laplace & - & 54.6\text{ }(4.0) & 57.4\text{ }(3.8) & 65.4\text{ }(2.8) & 68.0\text{ }(2.9) & 68.4\text{ }(0.5)  \\
PoissonMBO & - & 62.0\text{ }(5.7) & 67.2\text{ }(4.8) & 70.4\text{ }(2.9) & 72.1\text{ }(2.5) & 73.1\text{ }(2.7)  \\
CutSSL & - & 62.0\text{ }(5.7) & 67.2\text{ }(4.8) & 70.4\text{ }(2.9) & 72.1\text{ }(2.5) & 73.1\text{ }(2.7)  \\
SSLC/USLC (ours) & \textbf{61.6\text{ }(5.2)} & \textbf{65.8 \text{ }(4.1)} & \textbf{69.1\text{ }(3.2)} & \textbf{74.3\text{ }(2.3)} & \textbf{77.2\text{ }(2.6)} & \textbf{78.7\text{ }(2.3)}  \\ 
\bottomrule
\end{tabular}
\end{table}

We compare against other modern semi-supervised methods such as Laplace learning~\citep{ZZL03}, lazy random walks~\citep{Zhou04, Zhou04a}, weighted nonlocal laplacian (WNLL)~\citep{Shi17}, VolumeMBO~\citep{Jacobs18}, PoissonMBO~\citep{Calder20},
$p$-Laplace learning~\citep{Flores22}, and CutSSL~\citep{Holtz2024}. The results are summarized in Table~\ref{tabFashionMNIST} for FashionMNIST dataset, and Table \ref{tabCIFAR10} for CIFAR-10 dataset.  We observe that our method consistently outperforms other approaches in all cases. Additional experimental results on Planetoid are included in Table \ref{tabPlanetoid} in Appendix \ref{secAppendixAddExp}. Note that all the other methods operate only in the semi-supervised setting, where at least one labeled example per class is required, whereas our approach can be applied in a fully unsupervised manner.

\begin{table*}
\caption{Average accuracy and standard deviation (\%) on CIFAR-10 over 30 trials} \label{tabCIFAR10}
\vspace{2mm}
\centering
\begin{tabular} {lllllll}
\toprule
 \# Labels per class & 0 & 1  & 2 & 3 & 4 & 5 \\ 
 \midrule
Laplace & - & 10.4 (1.3) & 11.0 (2.1) & 11.6 (2.7) & 12.9 (3.9) & 14.1 (5.0)   \\
Random Walk & - & 36.4 (4.9) & 42.0 (4.4) & 45.1 (3.3) & 47.5 (2.9) & 49.0 (2.6) \\
VolumeMBO & - & 38.0 (7.2) & 46.4 (7.2) & 50.1 (5.7) & 53.3 (4.4) & 55.3 (3.8)  \\
WNLL & - & 16.6 (5.2) & 26.2 (6.8) & 33.2 (7.0) & 39.0 (6.2) & 44.0 (5.5)  \\
$p$-Laplace & - & 26.0 (6.7) & 35.0 (5.4) & 42.1 (3.1) & 48.1 (2.6) & 49.7 (3.8)  \\
PoissonMBO & - & 41.8 (6.5) & 50.2 (6.0) & 53.5 (4.4) & 56.5 (3.5) & 57.9 (3.2) \\
CutSSL & - & 41.8 (6.5) & 50.2 (6.0) & 53.5 (4.4) & 56.5 (3.5) & 57.9 (3.2) \\
SSLC/USLC (ours) & \textbf{48.2\text{ } (6.1)} & \textbf{51.2\text{ } (5.2)} & \textbf{57.1\text{ }(4.9)} & \textbf{59.7\text{ }(4.7)} & \textbf{61.3\text{ }(3.4)} & \textbf{63.5\text{ }(2.7)} \\
\bottomrule
\end{tabular}
\end{table*}

We further benchmark our method against several unsupervised clustering techniques such as $k$-means, GMM, DBSCAN~\citep{Ester96}, DPM Sampler~\citep{dinari2019distributed}, MoVB~\citep{hughes2013memoized}, DeepDPM~\citep{ronen2022deepdpm}, and observe that our approach attains superior performance across most baselines as presented in Table \ref{tab:unsupervised}. Note that for $k$-means and GMM, a prescribed number of clusters parameter $k$ is required. There are also recent contrastive learning–based clustering methods (e.g., SimCLR~\citep{chen2020simple}), which often offer a higher accuracy, but their computational demands are significantly greater. 
In contrast, the propose method is computationally more efficient and 
can be executed within seconds or minutes.


\begin{table}[t]
\caption{USLC accuracy (\%) compares against other unsupervised clustering methods \label{tab:unsupervised}} 
\label{tabDeep}
\centering
\begin{tabular} {lcc}
\toprule
 \# Clustering Methods  & FashionMNIST  & CIFAR-10\\ 
 \midrule
 K-means  & 60.0 & 41.0 \\
GMM & 58.0 & 40.0 \\
DBSCAN & 39.0 & 27.0\\
DPMsampler & 59.0 & 43.0 \\
MoVB & 55.0 & 39.0 \\
DeepDPM & \textbf{62.0} & 45.0 \\
USLC (ours) & 61.6 & \textbf{48.2} \\
\bottomrule
\end{tabular}
\end{table}

To test against more complex cases and demonstrate the robustness of our methods against outliers in the graph, we manually add some images consisting of random noise as outliers into datasets FashionMNIST and CIFAR-10. Specifically, we add the number of outliers equal to $10\%$ of the original dataset size. Each outlier image is generated by setting its pixel values equal to the standard Gaussian random noise. For example, as shown in Figure~\ref{figOutlier}, the last block consists entirely of outliers, exhibiting no community structure, while other blocks do. Such case is also called tight clustering in the literature~\citep{tseng2005tight,deng2024network}. By comparing Table~\ref{tabOutliers} with Table~\ref{tabFashionMNIST} and \ref{tabCIFAR10}, it further confirms that the performance of our proposed methods remains largely unaffected in the presence of these outliers.

The proposed methods exhibit strong efficiency and accuracy in settings characterized by low connectivity, namely sparse graphs, and low label rates, where only a few labeled nodes are available. This aligns with common real-world scenarios, as most practical graphs tend to be sparse. In contrast, when the graph is relatively dense or when a large number of seed nodes is provided, conventional algorithms or deep clustering techniques can be 
advantageous.

\begin{table}[t]
\caption{Average accuracy and standard deviation (\%) of SSLC/USLC by adding $10\%$ number of outliers} 
\label{tabOutliers}
\centering
\begin{tabular} {lcccccc}
\toprule
 \# Labels Per Class & 0 & 1  & 2 & 3 & 4 & 5\\ 
 \midrule
FashionMNIST & 55.6 (6.7) & 60.4 (6.2)  & 65.6 (5.9) & 69.9 (4.7) & 73.5 (4.0) & 74.8 (3.9)  \\
CIFAR-10 & 45.1 (7.9) & 47.3 (7.3)  & 51.6 (7.0) & 55.3 (5.9) & 58.1 (5.4) & 60.1 (4.7)  \\
\bottomrule
\end{tabular}
\end{table}

\section{Conclusion and Discussion} \label{secConclusion}

{\bf Conclusion and Limitation.~} We proposed a unified framework for extracting local clusters in both semi-supervised and unsupervised settings with theoretical guarantees. The methods require minimal label information in the semi-supervised case and no label information in the unsupervised case, achieving state-of-the-art results. Notably, the methods are very effective particularly in the low-label rates regime, and remain robust when outliers are presented. We point out two limitations of this work. First, our methods mostly work well for sparse graph. This is mainly due to sparse graph has clear clustering structure, meaning that it often have tightly connected clusters with fewer inter-cluster edges, making local algorithms efficient at detecting boundaries. Second, when the label rate is not low, our proposed method in the semi-supervised case do not exhibit advantages over other methods.

{\bf Broader Impact.~}  While local clustering algorithms offer valuable insights for applications like community detection and recommendation systems, their potential misuse raises ethical concerns, particularly in contexts like surveillance or discriminatory profiling. By identifying tightly connected groups in social networks, these techniques can inadvertently reinforce biases (e.g., via homophily) or enable repression when deployed without transparency or consent. The societal impact hinges on intent—clustering can empower communities when used responsibly, but it risks harming marginalized groups if applied for unchecked monitoring or targeting. Practitioners should weigh considerations like purpose, bias mitigation, and user awareness to ensure ethical deployment, acknowledging that even mathematically neutral tools carry real-world consequences.

\begin{ack}

Kang's research is supported in part by Simons Foundation grant number AWD-007751.
\end{ack}

\bibliographystyle{plainnat}
\bibliography{references}

@article{rahimi2024lifted,
  title={A lifted $\ell_1$ framework for sparse recovery},
  author={Rahimi, Yaghoub and Kang, Sung Ha and Lou, Yifei},
  journal={Information and Inference: A Journal of the IMA},
  volume={13},
  number={1},
  pages={iaad055},
  year={2024},
  publisher={Oxford University Press}
}

@article{kang2014supervised,
  title={Supervised and transductive multi-class segmentation using p-Laplacians and RKHS methods},
  author={Kang, Sung Ha and Shafei, Behrang and Steidl, Gabriele},
  journal={Journal of Visual Communication and Image Representation},
  volume={25},
  number={5},
  pages={1136--1148},
  year={2014},
  publisher={Elsevier}
}

@article{kang2011regularized,
  title={A regularized k-means and multiphase scale segmentation},
  author={Kang, Sung Ha and Sandberg, Berta and Yip, Andy M},
  journal={Inverse Problems and Imaging},
  volume={5},
  number={2},
  pages={407--429},
  year={2011},
  publisher={Inverse Problems and Imaging}
}

@article{CRT06,
  author={Candès, Emmanuel J. and Romberg, Justin and Tao, Terence},
  title={Robust uncertainty principles: Exact signal reconstruction from highly incomplete frequency information},
  journal={IEEE Transactions on Information Theory},
  volume={52},
  number={2},
  pages={489--509},
  year={2006},
  publisher={IEEE}
}

@article{D06,
  author={Donoho, David L.},
  title={Compressed sensing},
  journal={IEEE Transactions on Information Theory},
  volume={52},
  number={4},
  pages={1289--1306},
  year={2006},
  publisher={IEEE}
}

@article{HS10,
  author={Herman, Matthew A. and Strohmer, Thomas},
  title={General deviants: An analysis of perturbations in compressed sensing},
  journal={IEEE Journal of Selected Topics in Signal Processing},
  volume={4},
  number={2},
  pages={342--349},
  year={2010},
  publisher={IEEE}
}

@article{LM20,
  author={Lai, Ming-Jun and Mckenzie, Daniel},
  title={Compressive sensing for cut improvement and local clustering},
  journal={SIAM Journal on Mathematics of Data Science},
  volume={2},
  number={2},
  pages={368--395},
  year={2020},
  publisher={SIAM}
}

@article{LS23,
  author={Lai, Ming-Jun and Shen, Zhaiming},
  title={A compressed sensing based least squares approach to semi-supervised local cluster extraction},
  journal={Journal of Scientific Computing},
  volume={94},
  number={3},
  pages={63},
  year={2023},
  publisher={Springer}
}

@article{L16,
  author={Li, Haifeng},
  title={Improved analysis of SP and CoSaMP under total perturbations},
  journal={EURASIP Journal on Advances in Signal Processing},
  volume={2016},
  pages={1--6},
  year={2016},
  publisher={Springer}
}

@inproceedings{SLL23,
  author={Shen, Zhaiming and Lai, Ming-Jun and Li, Sheng},
  title={Graph-based Semi-supervised Local Clustering with Few Labeled Nodes},
  booktitle={Proceedings of the International Joint Conference on Artificial Intelligence (IJCAI)},
  pages={4190--4198},
  year={2023}
}

@article{hamel2024local,
  title={Local Clustering for Lung Cancer Image Classification via Sparse Solution Technique},
  author={Hamel, Jackson and Lai, Ming-Jun and Shen, Zhaiming and Tian, Ye},
  journal={arXiv preprint arXiv:2407.08800},
  year={2024}
}

@article{lai2020quasi,
  title={A quasi-orthogonal matching pursuit algorithm for compressive sensing},
  author={Lai, Ming-Jun and Shen, Zhaiming},
  journal={arXiv preprint arXiv:2007.09534},
  year={2020}
}

@article{V07,
  author={Von Luxburg, Ulrike},
  title={A tutorial on spectral clustering},
  journal={Statistics and Computing},
  volume={17},
  pages={395--416},
  year={2007},
  publisher={Springer}
}

@inproceedings{ZZL03,
  author={Zhu, Xiaojin and Ghahramani, Zoubin and Lafferty, John D.},
  title={Semi-supervised learning using Gaussian fields and harmonic functions},
  booktitle={Proceedings of the 20th International Conference on Machine Learning (ICML-03)},
  pages={912--919},
  year={2003}
}

@misc{ZG02,
  author={Zhu, Xiaojin and Ghahramani, Zoubin},
  title={Learning from labeled and unlabeled data with label propagation},
  year={2002},
  note={ProQuest number: information to all users}
}

@article{Xiao17,
  author={Xiao, Han and Rasul, Karim and Vollgraf, Roland},
  title={Fashion-MNIST: A novel image dataset for benchmarking machine learning algorithms},
  journal={arXiv preprint arXiv:1708.07747},
  year={2017}
}

@misc{Krizhevsky09,
  author={Krizhevsky, Alex and Hinton, Geoffrey and others},
  title={Learning multiple layers of features from tiny images},
  year={2009}
}

@inproceedings{Kloster14,
  author={Kloster, Kyle and Gleich, David F.},
  title={Heat kernel based community detection},
  booktitle={Proceedings of the 20th ACM SIGKDD International Conference on Knowledge Discovery and Data Mining},
  pages={1386--1395},
  year={2014},
  publisher={ACM}
}

@inproceedings{Andersen06,
  author={Andersen, Reid and Chung, Fan and Lang, Kevin},
  title={Local graph partitioning using pagerank vectors},
  booktitle={IEEE Symposium on Foundations of Computer Science},
  pages={475--486},
  year={2006},
  publisher={IEEE}
}

@article{Wilson14,
  author={Wilson, James D. and Wang, Simi and Mucha, Peter J. and Bhamidi, Shankar and Nobel, Andrew B.},
  title={A testing based extraction algorithm for identifying significant communities in networks},
  journal={The Annals of Applied Statistics},
  volume={8},
  pages={1853--1891},
  year={2014},
  publisher={Institute of Mathematical Statistics}
}

@article{Shi19,
  author={Shi, Pan and He, Kun and Bindel, David and Hopcroft, John E.},
  title={Locally-biased spectral approximation for community detection},
  journal={Knowledge-Based Systems},
  volume={164},
  pages={459--472},
  year={2019},
  publisher={Elsevier}
}

@inproceedings{Zhou04,
  author={Zhou, D. and Scholkopf, B.},
  title={Learning from labeled and unlabeled data using random walks},
  booktitle={Joint Pattern Recognition Symposium},
  pages={237--244},
  year={2004},
  publisher={Springer}
}

@inproceedings{Zhou04a,
  author={Zhou, D. and Bousquet, O. and Lal, T. N. and Weston, J. and Scholkopf, B.},
  title={Learning with local and global consistency},
  booktitle={Advances in Neural Information Processing Systems},
  pages={321--328},
  year={2004},
  publisher={MIT Press}
}

@article{Shi17,
  author={Shi, Z. and Osher, S. and Zhu, W.},
  title={Weighted nonlocal Laplacian on interpolation from sparse data},
  journal={Journal of Scientific Computing},
  volume={73},
  number={2-3},
  pages={1164--1177},
  year={2017},
  publisher={Springer}
}

@article{Jacobs18,
  author={Jacobs, M. and Merkurjev, E. and Esedoglu, S.},
  title={Auction dynamics: A volume constrained MBO scheme},
  journal={Journal of Computational Physics},
  volume={354},
  pages={288--310},
  year={2018},
  publisher={Elsevier}
}

@article{Flores22,
  author={Flores, Mauricio and Calder, Jeff and Lerman, Gilad},
  title={Analysis and algorithms for $\ell_p$-based semi-supervised learning on graphs},
  journal={Applied and Computational Harmonic Analysis},
  volume={60},
  pages={77--122},
  year={2022},
  publisher={Elsevier}
}

@inproceedings{Calder20,
  author={Calder, Jeff and Cook, Brendan and Thorpe, Matthew and Slepcev, Dejan},
  title={Poisson learning: Graph based semi-supervised learning at very low label rates},
  booktitle={Proceedings of the International Conference on Machine Learning (ICML)},
  pages={1306--1316},
  year={2020},
  publisher={PMLR}
}

@inproceedings{ZelnikManor04,
  author={Zelnik-Manor, Lihi and Perona, Pietro},
  title={Self-tuning spectral clustering},
  booktitle={Advances in Neural Information Processing Systems (NeurIPS)},
  volume={17},
  year={2004}
}

@inproceedings{Zhou05,
  author={Zhou, D. and Huang, J. and Scholkopf, B.},
  title={Learning from labeled and unlabeled data on a directed graph},
  booktitle={Proceedings of the 22nd International Conference on Machine Learning},
  pages={1036--1043},
  year={2005},
  publisher={ACM}
}

@inproceedings{Zhou04b,
  author={Zhou, D. and Weston, J. and Gretton, A. and Bousquet, O. and Scholkopf, B.},
  title={Ranking on data manifolds},
  booktitle={Advances in Neural Information Processing Systems (NeurIPS)},
  pages={169--176},
  year={2004},
  publisher={MIT Press}
}

@inproceedings{Ando07,
  author={Ando, R. K. and Zhang, T.},
  title={Learning on graph with Laplacian regularization},
  booktitle={Advances in Neural Information Processing Systems (NeurIPS)},
  pages={25--32},
  year={2007},
  publisher={MIT Press}
}

@inproceedings{Alaoui16,
  author={El Alaoui, A. and Cheng, X. and Ramdas, A. and Wainwright, M. J. and Jordan, M. I.},
  title={Asymptotic behavior of $l_p$-based Laplacian regularization in semi-supervised learning},
  booktitle={Conference on Learning Theory (COLT)},
  pages={879--906},
  year={2016}
}

@article{Calder19,
  author={Calder, J. and Slepcev, D.},
  title={Properly-weighted graph Laplacian for semi-supervised learning},
  journal={Applied Mathematics \& Optimization},
  year={2019},
  month={Dec}
}

@inproceedings{Zhou11,
  author={Zhou, X. and Belkin, M.},
  title={Semi-supervised learning by higher order regularization},
  booktitle={Proceedings of the Fourteenth International Conference on Artificial Intelligence and Statistics (AISTATS)},
  pages={892--900},
  year={2011}
}

@inproceedings{Veldt19,
  author={Veldt, Nate and Klymko, Christine and Gleich, David F.},
  title={Flow-based local graph clustering with better seed set inclusion},
  booktitle={Proceedings of the SIAM International Conference on Data Mining (SDM)},
  pages={378--386},
  year={2019},
  publisher={SIAM}
}

@inproceedings{Fountoulakis20,
  author={Fountoulakis, Kimon and Wang, Di and Yang, Shenghao},
  title={$p$-norm flow diffusion for local graph clustering},
  booktitle={Proceedings of the International Conference on Machine Learning (ICML)},
  pages={3222--3232},
  year={2020}
}

@inproceedings{Lang04,
  author={Lang, Kevin and Rao, Satish},
  title={A flow-based method for improving the expansion or conductance of graph cuts},
  booktitle={Proceedings of the International Conference on Integer Programming and Combinatorial Optimization (IPCO)},
  pages={325--337},
  year={2004}
}

@incollection{Nielsen16,
  author={Nielsen, Frank},
  title={Hierarchical clustering},
  booktitle={Introduction to HPC with MPI for Data Science},
  pages={195--211},
  year={2016}
}

@inproceedings{MacQueen67,
  author={MacQueen, J.},
  title={Classification and analysis of multivariate observations},
  booktitle={Proceedings of the Fifth Berkeley Symposium on Mathematical Statistics and Probability},
  pages={281--297},
  year={1967}
}

@inproceedings{Ester96,
  author={Ester, Martin and Kriegel, Hans-Peter and Sander, Jörg and Xu, Xiaowei},
  title={A density-based algorithm for discovering clusters in large spatial databases with noise},
  booktitle={Proceedings of the 2nd International Conference on Knowledge Discovery and Data Mining (KDD)},
  pages={226--231},
  year={1996}
}

@inproceedings{Ng01,
  author={Ng, Andrew and Jordan, Michael and Weiss, Yair},
  title={On spectral clustering: Analysis and an algorithm},
  booktitle={Advances in Neural Information Processing Systems (NeurIPS)},
  volume={14},
  year={2001}
}

@article{Tropp04,
  author={Tropp, Joel A.},
  title={Greed is good: Algorithmic results for sparse approximation},
  journal={IEEE Transactions on Information Theory},
  volume={50},
  pages={2231--2242},
  year={2004}
}

@article{Blumensath09,
  author={Blumensath, Thomas and Davies, Mike E.},
  title={Iterative hard thresholding for compressed sensing},
  journal={Applied and Computational Harmonic Analysis},
  volume={27},
  number={3},
  pages={265--274},
  year={2009}
}

@article{Feng21,
  author={Feng, Renzhong and Huang, Aitong and Lai, Ming-Jun and Shen, Zhaiming},
  title={Reconstruction of sparse polynomials via quasi-orthogonal matching pursuit method},
  journal={Journal of Computational Mathematics},
  year={2021}
}

@article{Dai09,
  author={Dai, Wei and Milenkovic, Olgica},
  title={Subspace pursuit for compressive sensing signal reconstruction},
  journal={IEEE Transactions on Information Theory},
  volume={55},
  number={5},
  pages={2230--2249},
  year={2009},
  month={May}
}

@article{tseng2005tight,
  title={Tight clustering: a resampling-based approach for identifying stable and tight patterns in data},
  author={Tseng, George C and Wong, Wing H},
  journal={Biometrics},
  volume={61},
  number={1},
  pages={10--16},
  year={2005},
  publisher={Wiley Online Library}
}

@inproceedings{deng2024network,
  title={Network tight community detection},
  author={Deng, Jiayi and Yang, Xiaodong and Yu, Jun and Liu, Jun and Shen, Zhaiming and Huang, Danyang and Cheng, Huimin},
  booktitle={Forty-first International Conference on Machine Learning},
  year={2024}
}

@article{SpielmanTeng2013,
  title={A Local Clustering Algorithm for Massive Graphs and Its Application to Nearly Linear Time Graph Partitioning},
  author={Spielman, Daniel A. and Teng, Shang-Hua},
  journal={SIAM Journal on Computing},
  volume={42},
  number={1},
  pages={1--26},
  year={2013},
  publisher={SIAM},
  doi={10.1137/110853996}
}

@article{Andersen2016,
  title={Almost Optimal Local Graph Clustering Using Evolving Sets},
  author={Andersen, Reid and Oveis Gharan, Shayan and Peres, Yuval and Trevisan, Luca},
  journal={Journal of the ACM (JACM)},
  volume={63},
  number={2},
  pages={1--31},
  year={2016},
  publisher={ACM},
  doi={10.1145/2856033},
  note={Article No. 15}
}

@phdthesis{shen2024sparse,
  title   = {Sparse Solution Technique in Semi-Supervised Local Clustering and High Dimensional Function Approximation},
  author  = {Shen, Zhaiming},
  school  = {University of Georgia},
  year    = {2024},
  type    = {Ph.D. dissertation}}

@inproceedings{Holtz2024,
  author    = {Holtz, Chester and Pengwen Chen and Zhengchao Wan and Chung-Kuan Cheng and Gal Mishne},
  title     = {Continuous Partitioning for Graph-Based Semi-Supervised Learning},
  booktitle = {The Thirty-eighth Annual Conference on Neural Information Processing Systems},
  year      = {2024},
  note      = {NeurIPS 2024}
}

@inproceedings{Yang2016,
  author    = {Yang, Zhilin and William Cohen and Ruslan Salakhudinov},
  title     = {Revisiting Semi-Supervised Learning with Graph Embeddings},
  booktitle = {Proceedings of the 33rd International Conference on Machine Learning},
  series    = {ICML'16},
  pages     = {40--48},
  year      = {2016},
  publisher = {PMLR},
  volume    = {48},
  note      = {Also introduces the Planetoid benchmark with standardized splits for Cora, CiteSeer, and PubMed datasets.}
}

@article{abbe2018community,
  author  = {Abbe, Emmanuel},
  title   = {Community Detection and Stochastic Block Models: Recent Developments},
  journal = {Journal of Machine Learning Research},
  year    = {2018},
  volume  = {18},
  number  = {177},
  pages   = {1--86}
}

@inproceedings{chen2020simple,
  title     = {A Simple Framework for Contrastive Learning of Visual Representations},
  author    = {Chen, Ting and Kornblith, Simon and Norouzi, Mohammad and Hinton, Geoffrey},
  booktitle = {Proceedings of the 37th International Conference on Machine Learning},
  pages     = {1597--1607},
  year      = {2020},
  volume    = {119},
  series    = {Proceedings of Machine Learning Research},
  publisher = {PMLR}
}

@inproceedings{dinari2019distributed,
  title     = {Distributed MCMC Inference in Dirichlet Process Mixture Models Using Julia},
  author    = {Dinari, Or and Yu, Angel and Freifeld, Oren and Fisher, John},
  booktitle = {2019 19th IEEE/ACM International Symposium on Cluster, Cloud and Grid Computing (CCGRID)},
  pages     = {518--525},
  year      = {2019},
  publisher = {IEEE}
}

@inproceedings{hughes2013memoized,
  title     = {Memoized Online Variational Inference for Dirichlet Process Mixture Models},
  author    = {Hughes, Michael C. and Sudderth, Erik B.},
  booktitle = {Advances in Neural Information Processing Systems},
  volume    = {26},
  year      = {2013}
}

@inproceedings{ronen2022deepdpm,
  title     = {DeepDPM: Deep Clustering with an Unknown Number of Clusters},
  author    = {Ronen, Meitar and Finder, Shahaf E. and Freifeld, Oren},
  booktitle = {Proceedings of the IEEE/CVF Conference on Computer Vision and Pattern Recognition (CVPR)},
  pages     = {9861--9870},
  year      = {2022}
}








\newpage

\appendix

\section{Lemmas and deferred proofs}
\label{secAppendixOtherLemmas}
Let us introduce 
\begin{equation}
\mathbf{x}^*:=\argmin_{\mathbf{x}  \in \mathbb{R}^{|V|-|U|}} \{\|L^{in}_{V\setminus U}\mathbf{x}- \mathbf{y}^{in}\|_2: \|\mathbf{x}\|_0\leq \hat{n}_1-|U|\},
\end{equation}
where $\mathbf{y}^{in}:=-\sum_{i\in V\setminus U}\ell^{in}_i=L^{in}\mathbf{1}_{V\setminus U}$, and 
\begin{equation}     
\mathbf{x}^\# := \argmin_{\mathbf{x}  \in \mathbb{R}^{|V|-|U|}} \{\|L_{V\setminus U}\mathbf{x}- \mathbf{y}\|_2: \|\mathbf{x}\|_0\leq \hat{n}_1-|U|\},
\end{equation}
where $\mathbf{y}:=-\sum_{i\in V\setminus U}\ell_i=L\mathbf{1}_{V\setminus U}$. The solution $\mathbf{x}^*$ gives the underlying true node indices associated with $L^{in}$, while $\mathbf{x}^\#$ gives the algorithmic output node indices associated with $L$.

We highlight the following Lemma~\ref{IndAna}, which is a crucial step in establishing the convergence result in LCE. We relax the assumption from $\|\Delta L\|_2=o(n^{-1/2})$ (see Lemma \ref{IndAnaLCE} in Appendix \ref{secAppendixLCE}) \citep{SLL23} to $\|\Delta L\|_2=o(1)$.

\begin{lemma} \label{IndAna}
Let $\Delta L:=L-L^{in}$. Suppose $U\subset C_1$ and $0.1|C_1|<|U|<0.9|C_1|$. Suppose further that $\|\Delta L\|_2=o(1)$ and $\left(\delta_{3(|C_1|-|U|)}(L)\right)^{\log(n)}=o(1)$. Then
\begin{equation}
\frac{\|\mathbf{x}^\#-\mathbf{x}^*\|_2}{\|\mathbf{x}^*\|_2}= o(1).
\end{equation}
\end{lemma}

Lemma \ref{IndAna} establishes the convergence result between $\mathbf{x}^\#$ and $\mathbf{x}^*$ when $\Delta L$ is small, which is the key step for establishing the correctness of our algorithms in Theorem \ref{thmSSLC}.

\begin{lemma} \label{CCMIV}
Suppose $G$ satisfies Assumptions~\ref{assump1} - \ref{assump3}. Let $v$ be any sampled node with $v\in C_s$ for some $s\geq 1$. Then, as $n$ gets large, the cardinality of the set $S:=\{i\in C_r: {\rm comembership}(\mathbf{1}_{C^\#})_{v,i}={\rm comembership}(\mathbf{1}_{C_s})_{v,i}\}$ satisfies
\[  \begin{cases} 
    |S| \geq (1-o(1))n_{\min}, & \text{if}\text{ }\text{ } r=s, \\
      |S|\leq o(n_{\min}), & \text{otherwise.} \\
   \end{cases}
\]
\end{lemma}


\subsection{Proof of Lemma~\ref{IndAna}}
\begin{proof} [Proof of Lemma~\ref{IndAna}]
First, since $\|L-L^{in}\|_2=o(1)$, we have

\[\frac{\|\Delta\mathbf{y}\|_2}{\|\mathbf{y}\|_2}=\frac{\|\mathbf{y}-\mathbf{y}^{in}\|_2}{\|\mathbf{y}\|_2}=\frac{\|(L-L^{in})\mathbf{1}_{V\setminus U}\|_2}{\|L\mathbf{1}_{V\setminus U}\|_2}\leq\frac{o(1)\sqrt{n}}{\|L\mathbf{1}_U\|_2}.\]
By Assumption~\ref{assump2},, the cluster $C_1$ is on the same order as the size of the graph $n$, hence $U\subset C_1$ is also on the same order as $n$. By Assumption~\ref{assump3}, all the nodes asymptotically have the same degree, therefore $\|L\mathbf{1}_{U}\|_2=\Theta(\|\mathbf{1}_U\|_2)=\Theta(\sqrt{n})$. Hence \[\frac{\|\mathbf{y}-\mathbf{y}^{in}\|_2}{\|\mathbf{y}\|_2}\leq \frac{o(1)\sqrt{n}}{\Theta(\sqrt{n})}=o(1).\]
Therefore, the quantity $\epsilon_{\mathbf{y}}=o(1)$ in Lemma~\ref{lemsp}. 

With the same assumption $\|L-L^{in}\|_2=o(1)$, it is not hard to see that, if the singular values of $L$ decays at a reasonable rate, then by applying the eigenvalue interlacing theorem, we have $\epsilon_{\Phi}^s=o(1)$ as well (by letting $\Phi=L$ in Lemma~\ref{lemsp} in Appendix~\ref{secAppendixCS}). With these, the second term on the right-hand-side of the inequality in Lemma~\ref{lemsp} satisfies \[\tau\frac{\sqrt{1+\delta_s}}{1-\epsilon^s_\Phi}(\epsilon^s_\Phi+\epsilon_{\mathbf{y}})\leq o(1).\] 

Furthermore, since $\mathbf{x}^\#$ is the output of Algorithm~\ref{alg1} after $m=O(\log n)$ iteration, we have $\rho^m=O\left((\delta_{3(|C_1|-|U|)}(L))^{\log n}\right)=o(1)$. Putting these together gives
\[\frac{\|\mathbf{x}^\#-\mathbf{x}^*\|_2}{\|\mathbf{x}^*\|_2}= o(1)\]
as desired.
\end{proof}

\subsection{Proof of Theorem~\ref{thmSSLC}}
\begin{proof} [Proof of Theorem~\ref{thmSSLC}]
     For any $s\geq 1$, let us case on whether $v\in C_s$ or not.
     If $v\in C_s$, then by Lemma \ref{IndAna} and Theorem~\ref{SymmDiff}, we have 
     \begin{equation*}
        |C^\#\triangle (\tilde{C}_s\cap C^\#)|\leq |C^\#\triangle (C_s\cap C^\#)|+|(C_s\triangle\tilde{C}_s)\cap C^\#|\leq o(|C^\#|)+o(|C^\#|)=o(|C^\#|).  
     \end{equation*}
     Therefore, $|\tilde{C}_s\cap C^\#|\geq |C^\#|-o(|C^\#|)=(1-o(1))|C^\#|$.
     
     If $v\in C_t, t\neq s$, then $|C_s\cap C^\#|\leq o(|C^\#|)$. We have
     \begin{equation*}
         |C^\#\triangle (\tilde{C}_s\cap C^\#)|\geq |C^\#\triangle (C_s\cap C^\#)|-|(C_s\triangle\tilde{C}_s)\cap C^\#|\geq |C^\#|-o(|C^\#|)-o(|C^\#|)=|C^\#|-o(|C^\#|).
     \end{equation*}
    Therefore, $|\tilde{C}_s\cap C^\#|\leq |C^\#| - ( |C^\#|-o(|C^\#|))=o(|C^\#|)$.   
    
\end{proof}

\subsection{Proof of Lemma~\ref{CCMIV}}
\begin{proof} [Proof of Lemma~\ref{CCMIV}]
By Lemma \ref{IndAna} and Theorem~\ref{SymmDiff}, we have $|C^\#\triangle (C_s\cap C^\#)|\leq o(C^\#)= o(n_{\min})$. Therefore, 
\begin{equation*}
  |C^\#\triangle C_s| = |C^\#\triangle (C_s\cap C^\#)| + |C_s\setminus C^\#| \leq o(n_{\min}) + n_s-n_{\min}.  
\end{equation*}
Hence
\begin{align*}
    \left|\{i\in C_s: {\rm comembership}(\mathbf{1}_{C^\#})_{v,i}\neq {\rm comembership}(\mathbf{1}_{C_s})_{v,i}\}\right| &= |(C^\#\triangle C_s)\cap C_s| \leq |C^\#\triangle C_s| \\
    &\leq o(n_{\min}) + n_s-n_{\min}.
\end{align*}
So we conclude
\begin{align*}
    \left|\{i\in C_s: {\rm comembership}(\mathbf{1}_{C^\#})_{v,i}= {\rm comembership}(\mathbf{1}_{C_s})_{v,i}\}\right| &\geq  n_s - ( o(n_{\min}) + n_s-n_{\min}) \\
    & = n_{\min}-o(n_{\min}).
\end{align*}
As $C_s\cap C_t=\emptyset$, we have $|C^\#\cap C_t|\leq o(n_{min})$. Therefore,
\begin{align*}
    \left|\{i\in C_t: {\rm comembership}(\mathbf{1}_{C^\#})_{v,i}= {\rm comembership}(\mathbf{1}_{C_t})_{v,i}\}\right| \leq o(n_{\min}).
\end{align*}
\end{proof}

\subsection{Proof of Proposition \ref{Mmatrix}}
\begin{proof} [Proof of Proposition \ref{Mmatrix}]
     By Lemma~\ref{CCMIV}, we have the following estimates for the entries in the comembership matrix $M$ asymptotically.

     For any $i=j\in C_s$, some $s\geq 1$, the entries in $M$ satisfies

     \begin{equation*}
         M_{i,i}\geq \frac{n_s}{n}\cdot \frac{(1-o(1))n_{\min}}{n_s}=(1-o(1))\frac{n_{\min}}{n}.
     \end{equation*}

    For any distinct $i,j\in C_s$, $s\geq 1$, by Assumption \ref{assump3}, the entries in $M$ satisfies 
    \begin{equation*}
         M_{i,j} \geq \frac{1}{n}\cdot\frac{n^2_{\min}(1-o(1))}{n}=(1-o(1))\frac{n^2_{\min}}{n^2}.
    \end{equation*}
    For any distinct $i,j$ such that $i\in C_s$, $j\in C_t$, $s\neq t$, the entries in $M$ satisfies 
    \begin{align*}
        M_{i,j} &\leq \frac{n_s}{n}\cdot \frac{(1-o(1))n_{\min}}{n_s}\cdot \frac{o(n_{\min})}{n-n_s}+\frac{n_t}{n-n_t}\cdot\frac{(1-o(1))n_{\min}}{n_t}\cdot\frac{o(n_{\min})}{n} \\
        &\leq \frac{2n^2_{\min}(1-o(1))\cdot o(1)}{n^2} + \frac{2n^2_{\min}(1-o(1))\cdot o(1)}{n^2} \\
        & = o\left(\frac{4n^2_{\min}(1-o(1))}{n^2}\right).
    \end{align*}
\end{proof}

\subsection{Proof of Theorem~\ref{thmUSLC}}
\begin{proof} [Proof of Theorem~\ref{thmUSLC}]
By direct computation, we can choose $\delta=\frac{1}{2}(\frac{n_{min}}{n})^2$ such that 
    \begin{equation*}
        (1-o(1))\frac{n_{\min}}{n}>\delta>o\left(\frac{4n^2_{\min}(1-o(1))}{n^2}\right).
    \end{equation*}
     and 
    \begin{equation*}
        (1-o(1))\frac{n_{\min}^2}{n^2}>\delta>o\left(\frac{4n^2_{\min}(1-o(1))}{n^2}\right).
    \end{equation*}
Hence, for any $i\in C_s$ satisfies $M_{v,i}>\delta$, and any $i\in C_t, t\neq s$, satisfies $M_{v,i}<\delta$. So we conclude $C_s^\#=C_s$.
\end{proof}

\begin{remark}
    Note that All the statements in Theorem~\ref{thmSSLC} and \ref{thmUSLC}, Lemma~\ref{IndAna} and \ref{CCMIV}, and Proposition \ref{Mmatrix}, are asymptotic statements. The $o(\cdot)$ notation is with respect to $n\to\infty$.
\end{remark}

\section{Additional Algorithms} 

\label{secAdditionalAlg}

\begin{algorithm}[h]
\caption{Semi-supervised Local Clustering (SSLC) for Multiple Clusters \label{alg6}}
\begin{algorithmic}[1]
\REQUIRE The adjacency matrix $A$ of an underlying graph $G$, the number of cluster $k$, the initial seed(s) $\Gamma_s$ for each cluster, the estimated size ${n}_s$ for each cluster, $s=1,\cdots,k$, the number of resampling iteration $\ell$
\ENSURE  desired output clusters $C_1^{\#},\cdots,C_k^{\#}$

\FOR{$s=1:k$}
\STATE{$\tilde{C}_s \leftarrow{\rm LCE}(A,{n}_s,\Gamma_s)$}
\ENDFOR
\FOR{$i=1:\ell$}
\STATE{$v\leftarrow$ uniformly random sampled node from $G$}
\STATE{$C^{\#} \leftarrow{\rm LCE}(A,\min\{n_s\}_{s\geq 1},v)$}
\IF{$|\tilde{C}_s\cap C^{\#}| > 0.5|C^{\#}|$ for some $s$}
\STATE{$\Gamma_s\leftarrow \Gamma_s\cup\{v\}$}
\STATE{$\tilde{C}_s \leftarrow{\rm LCE}(A,{n}_s,\Gamma_s)$}
\ENDIF
\ENDFOR
\FOR{$s=1:k$}
\STATE{$C_s^{\#} \leftarrow \tilde{C}_s$ }
\ENDFOR
\end{algorithmic}
\end{algorithm}

    We use $C^\#$ to indicate the output of one of the LCE steps after sampling a random node $v$. Since $v$ can be from any cluster $s$, the notation $C^\#$ is independent of $s$.
    Due to Theorem \ref{thmSSLC}, the condition $|\tilde{C}_s\cap C^{\#}| > 0.5|C^{\#}|$ can only happen for one particular $s$. Therefore, Step 7 in Algorithm \ref{alg6} will only add each $v$ to one of the clusters.

\begin{remark} \label{rmktime4SSLC}
    One key advantage of Algorithm~\ref{alg6} is its ability to simultaneously identify all clusters, while previous method such as LCE can only detect one cluster at a time. This characteristic provides greater flexibility and makes Algorithm~\ref{alg6} significantly more efficient for practical applications, particularly when the number of clusters is large. This advantage is also reflected in the first and second columns of Figure~\ref{figSBM}.
\end{remark}

\section{Review on Local Cluster Extraction (LCE)} \label{secAppendixLCE}


The following result is central to our proposed sparse solution based local clustering method. We omit its proof by referring to \citet{V07}.

\begin{lemma} \label{kernelthm}
Let G be an undirected graph with non-negative weights. The multiplicity $k$ of the eigenvalue zero of the graph Laplacian $L:=I-D^{-1}A$ equals to the number of connected components  $C_1, C_2, \cdots, C_k$ in $G$. Further, the indicator vectors $\textbf{1}_{C_1}, \cdots, \textbf{1}_{C_k}\in\mathbb{R}^n$ on these components span the kernel of $L$.
\end{lemma}

For the convenience of our discussion, let us introduce more notations . For a graph $G=(V,E)$ with certain underlying community structure, it is convenient to write $G=G^{in}\cup G^{out}$, where $G^{in}=(V, E^{in})$, $G^{out}=(V, E^{out})$. Here $E^{in}$ is the set of all intra-connection edges within the same community (cluster), $E^{out}$ is the remaining edges in $E$. Further, we use $A^{in}$ and $L^{in}$ to denote the adjacency matrix and Laplacian matrix associated with $G^{in}$ respectively. 
In practice, we do not guarantee knowledge of the cluster to which each individual vertex belongs, meaning that $A^{in}$ and $L^{in}$ are not directly accessible. Instead, we only have access to $A$ and $L$.  A summary of the notations being used throughout this paper is included in Table~\ref{Notation} in Appendix~\ref{secAppendixNotations}.



Let us first briefly introduce the idea Local Cluster Extraction (LCE), which applies the idea of compressive sensing (or sparse solution)  technique to extract the target cluster in a semi-supervised manner. See also \citep{LM20,LS23,SLL23,shen2024sparse} for references. 

Suppose the graph $G$ consists of $k$ connected equal-size
components $C_1, \cdots, C_k$, in other words, there is no edge connection between different clusters, i.e., $L=L^{in}$. For illustration purpose, let us permute the matrix according to the membership of each node, then $L^{in}$ is in a block diagonal form i.e., all the off-diagnoal blocks equal to zero:
\begin{equation} \label{Lin}
L=L^{in} = \left( \begin{array}{cccc}
L^{in}_{C_1} &  &  &   \\
& L^{in}_{C_2} &  &   \\
&  & \ddots &  \\
&  &  & L^{in}_{C_k}  \\
\end{array} \right).
\end{equation}

Suppose that the target cluster is $C_1$. By Lemma \ref{kernelthm}, $\{\mathbf{1}_{C_1},\cdots, 
\mathbf{1}_{C_k}\}$ forms a basis of the kernel $W_0$ of $L$. Note that all the $\mathbf{1}_{C_i}$ 
have disjoint supports, so for $\mathbf{w}\in W_0$ and $\mathbf{w}\neq 0$, we have  write 
\begin{equation}
\mathbf{w}=\sum_{i=1}^{k} \alpha_i \mathbf{1}_{C_i}
\end{equation}
with some $\alpha_i\neq 0$. If a prior knowledge is given that first node $v_1\in C_1$, then the first cluster $C_1$ can be found by solving
\begin{equation} \label{constrainminnolabel}
\min||\mathbf{w}||_0 \quad \text{s.t.} \quad L^{in}\mathbf{w}=\mathbf{0} \quad \text{and} \quad w_1=1,
\end{equation}
which gives the solution $\mathbf{w}=\mathbf{1}_{C_1}\in\mathbb{R}^n$ as desired. It is worthwhile to note that  (\ref{constrainminnolabel}) is equivalent to
\begin{equation} \label{constrainminnolabel2}
\min||\mathbf{w}_{-1}||_0 \quad \text{s.t.} \quad L_{-1}^{in}\mathbf{w}_{-1}=-\ell_1,
\end{equation}
where $L_{-1}^{in}$ is a submatrix of $L^{in}$ with first column being removed, $\ell_1$ is the first column of $L^{in}$. The solution to (\ref{constrainminnolabel2}) is $\mathbf{w}_{-1}=\mathbf{1}_{C_1\setminus {v_1}}\in\mathbb{R}^{n-1}$ which encodes the same index information as the solution to (\ref{constrainminnolabel}). The benefit of formulation (\ref{constrainminnolabel2}) is that it can be solved by compressive sensing algorithms. Note that the sparse solution from solving (\ref{constrainminnolabel}) or (\ref{constrainminnolabel2}) always gives the indices corresponding to the target cluster. Therefore, the permutation does not affect our analysis and clustering result. The above procedure is named CS-LCE or LCE in \citep{SLL23}. We  summarize LCE as Algorithm \ref{alg1} in Appendix \ref{secAppendixLCE} and provide brief analysis. We also briefly introduce the compressive sensing and sparse solution technique in Appendix \ref{secAppendixCS}.

The following results are established in \citep{SLL23},  and we refer the proofs to \citep{SLL23}.
Let $L^{in}$ be the matrix as in (\ref{Lin}), then the true solution $\mathbf{x}^*$ which encodes the local cluster information is 
\begin{equation} \label{noiseless}
\mathbf{x}^*:=\argmin_{\mathbf{x}  \in \mathbb{R}^{|V|-|U|}} \{\|L^{in}_{V\setminus U}\mathbf{x}- \mathbf{y}^{in}\|_2: \|\mathbf{x}\|_0\leq \hat{n}_1-|U|\}
\end{equation}
where $\mathbf{y}^{in}=-\sum_{i\in V\setminus U}\ell^{in}_i=L^{in}\mathbf{1}_{V\setminus U}$.

\begin{theorem} \label{correctness}
Suppose $U\subset C_1$ and $0.1|C_1|<|U|<0.9|C_1|$. Then $\mathbf{x}^*=\mathbf{1}_{C_1\setminus U}\in\mathbb{R}^{|V|-|U|}$ is the unique solution to (\ref{noiseless}).
\end{theorem}

\begin{lemma} \label{IndAnaLCE}
Let $\Delta L:=L-L^{in}$. Suppose $U\subset C_1$ and $0.1|C_1|<|U|<0.9|C_1|$. Suppose further that $\|\Delta L\|_2=o(n^{-\frac{1}{2}})$ and $\delta_{3(|C_1|-|U|)}(L)=o(1)$. Then
\begin{equation}
\frac{\|\mathbf{x}^\#-\mathbf{x}^*\|_2}{\|\mathbf{x}^*\|_2}= o(1).
\end{equation}
\end{lemma}



\begin{theorem} \label{SymmDiff}
Under the same assumption as Lemma~\ref{IndAnaLCE}. Then
\begin{equation}
    \frac{|C_1\triangle C_1^{\#}|}{|C_1|}\leq o(1).
\end{equation}
\end{theorem}


We make use of Lemma \ref{IndAna} (improved version of Lemma \ref{IndAnaLCE}) and Theorem \ref{SymmDiff} in our main proofs.  The LCE algorithm is summarized in Algorithm \ref{alg1}.

\begin{algorithm}
\caption{Compressive Sensing of Local Cluster Extraction (CS-LCE or LCE \citep{SLL23}) \label{alg1}}
\begin{algorithmic}[1]
\REQUIRE Adjacency matrix $A$, the seed(s) set $\Gamma\subset C$ for some target cluster $C$, a known estimated size $\hat{n}_1\approx |C_1|$, a fixed number of random walk depth $t\in\mathbb{Z}^+$, sparsity parameter $\gamma\in [0.1, 0.5]$, random walk threshold parameter $\epsilon\in (0,1)$,  rejection parameter $R\in [0.1, 0.9]$.
\ENSURE the output target cluster $C^\#$
\STATE{Compute $L=I-D^{-1}A$, $P=AD^{-1}$, $\mathbf{v}^{0}=D\mathbf{1}_{\Gamma}$ and $\mathbf{v}^{(t)}=P^t\mathbf{v}^{(0)}$}
\STATE{Define $\Omega={\mathcal{L}}_{(1+\epsilon)\hat{n}}(\mathbf{v}^{(t)})$}
\STATE{Let $U$ be the set of column indices of $\gamma\cdot|\Omega|$ smallest components of the vector $|L_{\Omega}^{\top}|\cdot|L\mathbf{1}_{\Omega}|$.}
\STATE{Set $\mathbf{y}:=-\sum_{i\in V\setminus U}\ell_i=L\mathbf{1}_{V\setminus U}$. Let $\mathbf{x}^\#$ be the solution to
        \begin{equation}     \label{inalglsqremoveT}
        \argmin_{\mathbf{x}  \in \mathbb{R}^{|V|-|U|}} \{\|L_{V\setminus U}\mathbf{x}- \mathbf{y}\|_2: \|\mathbf{x}\|_0\leq \hat{n}_1-|U|\}
        \end{equation}
        obtained by using $O(\log n)$ iterations of \emph{Subspace Pursuit} \cite{Dai09}.}
\STATE{Let $C^{\#} = \{i: \mathbf{x}_i^{\#}>R\}\cup U$}
\end{algorithmic}
\end{algorithm}

The notation $\mathcal{L}$ is defined as
\[\mathcal{L}_s(\mathbf{v}):=\{i\in[n]: \text{$v_i$ among $s$ largest-in-magnitude entries in $\mathbf{v}$}\}.\]

\section{Background on Compressive Sensing and Sparse Solution Technique} \label{secAppendixCS}

The concept of compressive sensing (also called compressed sensing or sparse sampling) emerged from fundamental challenges in signal acquisition and efficient data compression. At its core, it addresses the inverse problem of recovering a sparse (or compressible) signal from a small number of noisy linear measurements:
\begin{equation} \label{CS1}
    \min_{\mathbf{x}\in\mathbb{R}^n} \|\mathbf{x}\|_0 \quad s.t. \quad \|\Phi\mathbf{x}-\mathbf{y}\|_2\leq\epsilon,
\end{equation} 
where $\Phi\in\mathbb{R}^{m\times n}$ is called sensing matrix (usually underdetermined),  $\mathbf{y}\in\mathbb{R}^n$ is called measurement vector, and the ``zero quasi-norm" $\|\cdot\|_0$ counts the number of nonzero components in a vector.
Among the key contributors, Donoho \citep{D06} and Candès, Romberg and Tao \citep{CRT06} are widely credited with being the first to explicitly introduce this concept and make it popular. Since then, two families of approaches such as thresholding type of algorithms~\citep{Blumensath09} and greedy type of algorithms~\citep{Tropp04,Feng21,lai2020quasi} have been developed based on the idea of compressive sensing. One particular type of greedy algorithms that has garnered our attention is the subspace pursuit \citep{Dai09}. 

One of the reasons behind the remarkable usefulness of compressive sensing lies in its robustness against errors, including both additive and multiplicative types. More precisely, suppose we know $\mathbf{y}=\Phi\mathbf{x}^*$ where $\mathbf{y}$ is the exact measurement of the acquired signal and $\Phi$ is the exact measurement of the sensing matrix. However, we may only be able to access to the noisy version $\tilde{\mathbf{y}}=\mathbf{y}+\Delta\mathbf{y}$ and $\tilde{\Phi}=\Phi+\Delta\Phi$. In such case, we can still approximate the solution $\mathbf{x}^*$ well from the noisy measurements $\tilde{\mathbf{y}}$ and $\tilde{\Phi}$, as explained in the work of \citep{HS10}. A unified framework of lifted $\ell_1$ form is explored in \citep{rahimi2024lifted}.One crucial concept which is often employed in the compressive sensing algorithm is called Restricted Isometry Property (RIP).

\begin{definition}
[Restricted Isometry Property]
\label{def:RIP}
Let $0< s <m$ be an integer, and sensing matrix $\Phi\in\mathbb{R}^{m\times n}$. Suppose there exists a constant $\delta_s>0$ such that
\begin{equation} \label{RIP}
    (1-\delta_s)\|\mathbf{x}\|_2^2\leq\|\Phi\mathbf{x}\|_2^2\leq (1+\delta_s)\|\mathbf{x}\|_2^2
\end{equation}
for all $\mathbf{x}\in\mathbb{R}^n$ with $\|\mathbf{x}\|_0\leq s$. Then the matrix $\Phi$ is said to have the Restricted Isometry Property (RIP) of order $s$. The smallest constant $\delta_s(\Phi)$ which makes (\ref{RIP}) hold is called the Restricted Isometry Constant (RIC) of $\Phi$.
\end{definition}

For Subspace Pursuit algorithm, we have the result in Lemma \ref{lemsp} \citep{L16}. 

\begin{lemma} \label{lemsp}
   Let $\mathbf{x}^*$, $\mathbf{y}$, $\tilde{y}$, $\Phi$, $\tilde{\Phi}$ be as defined above, and for any $t\in [n]$, let $\delta_s:=\delta_s(\tilde{\Phi})$. Suppose that $\|\mathbf{x}^*\|_0\leq s$. Define the following constants:
   \begin{equation*}
       \epsilon_{\mathbf{y}}:=\|\mathbf{\Delta\mathbf{y}}\|_2/\|\mathbf{y}\|_2 \quad \text{and} \quad \epsilon_{\Phi}^s:=\|\Delta\Phi\|_2^{(s)}/\|\Phi\|_2^{(s)}
   \end{equation*}
   where $\|M\|_2^{(s)}:=\max\{\|M_S\|_2: S\subset [n], \#(S)=s \}$ for any matrix $M$. Define further: 
   \begin{equation*}
    \rho:=\frac{\sqrt{2\delta_{3s}^2(1+\delta_{3s}^2)}}{1-\delta_{3s}^2} \quad\text{and}\quad
    \tau: = \frac{(\sqrt{2}+2)\delta_{3s}}{\sqrt{1-\delta_{3s}^2}}(1-\delta_{3s})(1-\rho)+\frac{2\sqrt{2}+1}{(1-\delta_{3s})(1-\rho)}.
   \end{equation*}
   Assume that $\delta_{3s}<0.4859$ and let $\mathbf{x}^{(m)}$ be the output of Algorithm \ref{alg1} after $m$ iterations. Then:
   \begin{equation*}
       \frac{\|\mathbf{x}^*-\mathbf{x}^{(m)}\|_2}{\|\mathbf{x}^*\|_2}\leq \rho^m+\tau\frac{\sqrt{1+\delta_s}}{1-\epsilon^s_\Phi}(\epsilon^s_\Phi+\epsilon_{\mathbf{y}}).
   \end{equation*}
\end{lemma}

\section{Implementation Details} \label{secAppendixImplementation}

\subsection{Constructing KNN Graphs}
Let $\mathbf{x}_i\in\mathbb{R}^n$ be the vectorization of an image or the feature extracted from an image, we define the following affinity matrix of the $K$-NN auxiliary graph based on Gaussian kernel: 
\[ A_{ij} = \begin{cases} 
e^{-\|\mathbf{x}_i-\mathbf{x}_j\|^2/\sigma_i\sigma_j} & \text{if} \quad \mathbf{x}_j\in NN(\mathbf{x}_i,K), \\
0 & \text{otherwise}. 
\end{cases}
\]
Note that similar construction has also appeared among others \citep{ZelnikManor04,Jacobs18,Calder20}. 
To construct high-quality graphs, we trained autoencoders to extract key features from the image data, we adopt the same parameters as \citep{Calder20} for training autoencoders to obtain these features.

The notation $NN(\mathbf{x}_i,K)$ indicates the set of $K$-nearest neighbours of $\mathbf{x}_i$, and $\sigma_i:=\|\mathbf{x}_i-\mathbf{x}^{(r)}_i\|$ where $\mathbf{x}^{(r)}_i$ is the $r$-th closest point of $\mathbf{x}_i$. Note that the above $A_{ij}$ is not necessary symmetric, so we consider $\tilde{A}_{ij}=A^T A$ for symmetrization. Alternatively, one may also want to consider $\tilde{A}=\max\{A_{ij}, A_{ji}\}$ or $\tilde{A}=(A_{ij}+A_{ji})/2$. We use $\tilde{A}$ as the input adjacency matrix for our algorithms. In our experiments, the local scaling parameters are chosen to be $K=15$, $r=10$ for all three real image datasets.

\subsection{Parameters of Synthetic Data Generating and Algorithms}
In all implementations where the LCE are applied, we sampled the seeds $\Gamma_i$ uniformly from each $C_i$ for all the implementations. We fix the rejection parameter $R=0.1$, the random walk depth $t=3$, random walk threshold parameter $\epsilon=0.8$, and the removal set parameter $\gamma=0.2$ for all experiments. 

For synthetic data, the symmetric stochastic block model consists of three equal size clusters of $n_1 = 200,400,600,800,1000$ with connection probability $p=5\log (3n_1)/(3n_1)$ and $q=\log (3n_1)/(3n_1)$. The general stochastic block model consists of clusters' sizes as $\mathbf{n}=(n_1,2n_1,5n_1)$ where $n_1$ is chosen from $200,400,600,800,1000$, and the connection probability has the matrix form 
\[
P=[P_{11},P_{12},P_{13};P_{21},P_{22},P_{23};P_{31},P_{32},P_{33}]\] 
with $P_{11}=\log^2(8n_1)/(6n_1)$, $P_{22}=\log^2(8n_1)/(12n_1)$, $P_{33}=\log^2(8n_1)/(30n_1)$, $P_{12}=P_{21}=P_{13}=P_{31}=P_{23}=P_{32}=\log(8n_1)/(6n_1)$.  In the implementation of SSLC on synthetic data, we choose the iteration number $L=60$ in the symmetric case and $L=90$ in the nonsymmetric case.

In the implementation of SSLC on FashionMNIST and CIFAR-10 (both with and w/o outliers), we choose the iteration number to be $L=50k$ where $k$ is the number of seeds, i.e., $k=1,2,3,4,5$. In the implementation of USLC on FashionMNIST and CIFAR-10 (both with and w/o outliers), we choose the iteration number to be $L=1000$, and we set $\delta=0.01$ for FashionMNIST and $\delta=0.005$ for CIFAR-10.

\subsection{Geometric Dataset} \label{secAppendixGeo}

\paragraph{Three Lines}
We generate three parallel lines in the $x$-$y$ plane defined by:
\begin{itemize}
    \item Line 1: $y = 0$ with $x \in [0, 6]$
    \item Line 2: $y = 1$ with $x \in [0, 6]$
    \item Line 3: $y = 2$ with $x \in [0, 6]$
\end{itemize}
For each line, we sample \textbf{1,200 points} uniformly at random. The embedding into $\mathbb{R}^{100}$ is performed as:
\begin{enumerate}
    \item \textbf{Zero-padding}: $(x, y) \mapsto (x, y, 0, \ldots, 0) \in \mathbb{R}^{100}$
    \item \textbf{Noise addition}: Each coordinate $z_i$ is perturbed as $z_i \leftarrow z_i + \epsilon_i$ where $\epsilon_i \sim \mathcal{N}(0, 0.15)$
\end{enumerate}

\paragraph{Three Circles}
We construct three concentric circles with:
\begin{itemize}
    \item Circle 1: Radius $r = 1.0$ (500 points)
    \item Circle 2: Radius $r = 2.4$ (1,200 points)
    \item Circle 3: Radius $r = 3.8$ (1,900 points)
\end{itemize}
Points are sampled uniformly along each circle, totaling \textbf{3,600 points}. The $\mathbb{R}^{100}$ embedding follows the same zero-padding and noise injection procedure as above.

\paragraph{Three Moons}
Three semicircular clusters are generated with:
\begin{itemize}
    \item Moon 1: Upper semicircle, radius $1.0$, centered at $(0, 0)$ (1,200 points)
    \item Moon 2: Lower semicircle, radius $1.5$, centered at $(1.5, 0.4)$ (1,200 points)
    \item Moon 3: Upper semicircle, radius $1.0$, centered at $(3, 0)$ (1,200 points)
\end{itemize}
Each dataset uses identical embedding:
\[
(x,y) \mapsto (x,y,0,\ldots,0) + \boldsymbol{\epsilon}, \quad \boldsymbol{\epsilon} \sim \mathcal{N}(\mathbf{0}, 0.15^2\mathbf{I}_{100})
\]

\begin{figure}[h]
    \centering
    \begin{tabular}{ccc}
        \includegraphics[width=0.3\linewidth]{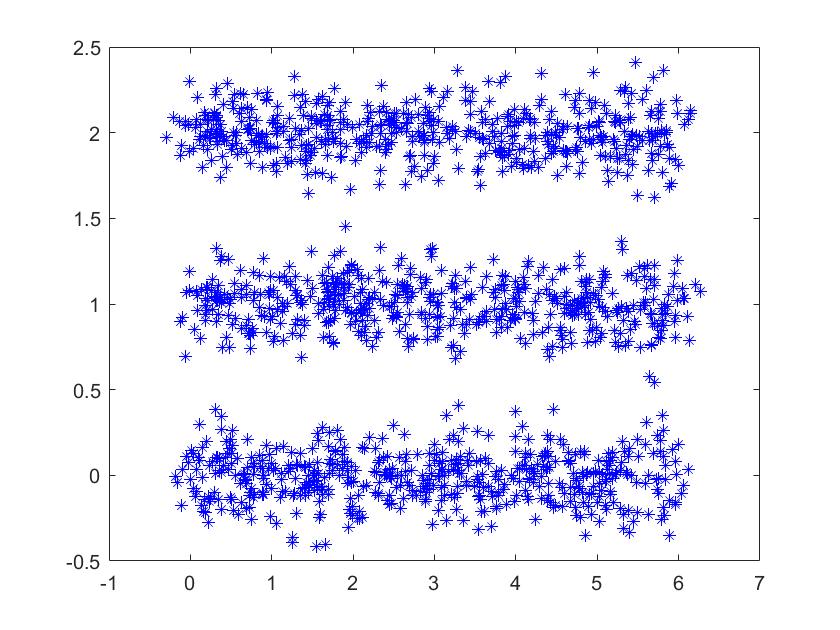}  
         & \includegraphics[width=0.3\linewidth]{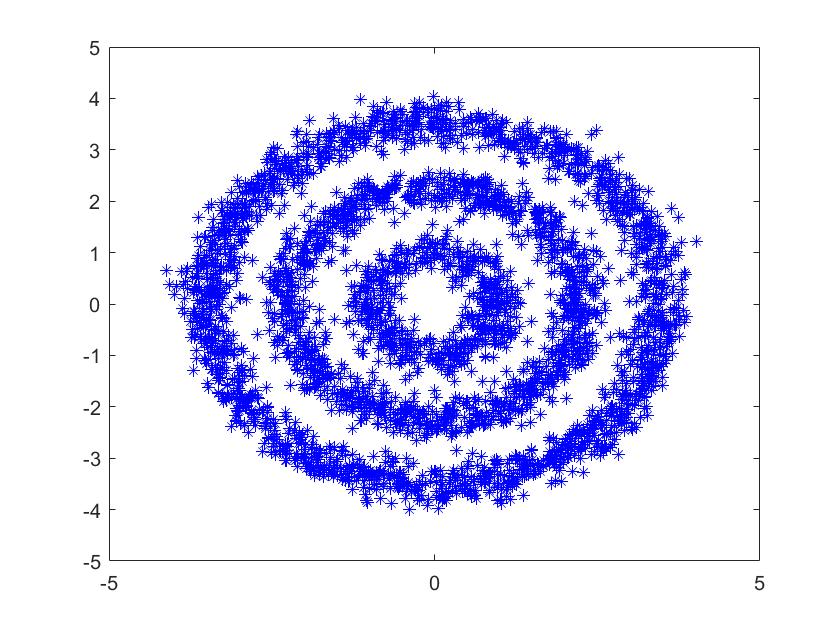} &
         \includegraphics[width=0.3\linewidth]{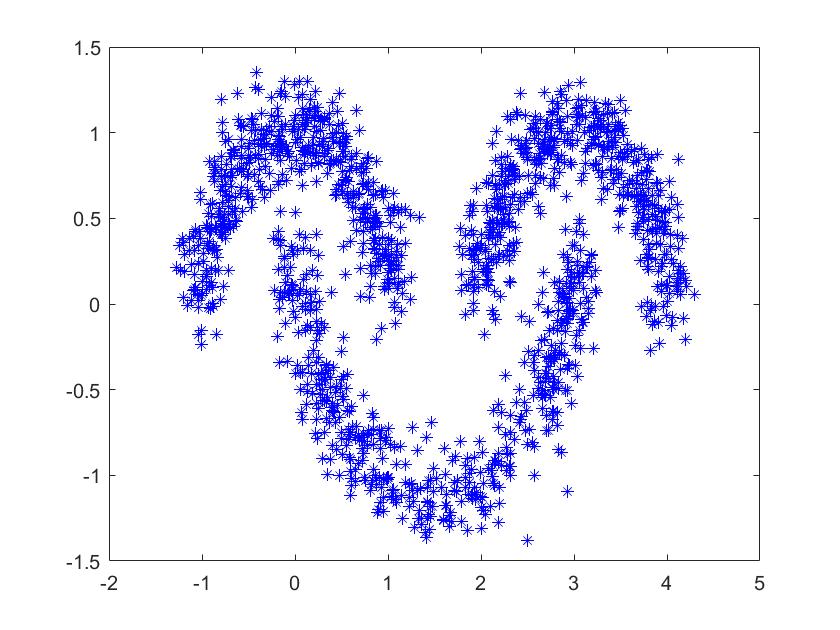}
    \end{tabular}
    \vspace{-1mm}
    \caption{2D Visulization of Geometric Dataset}
    \label{GeomData}
\end{figure}


\section{Additional Experimental Results} \label{secAppendixAddExp}

\begin{table}[h]
\centering
\caption{Jaccard Index of SSLC compared against the spectral norm of $\Delta L$ and SNR value on SSBM($n,k,p,q$) with $k=3$, $p=6\log n/n$, $q=\log n/n$.}
\label{DeltaLTable}
	\begin{tabular}{cccc}
	\toprule
           $n$ & Jaccard Index ($\%$) & Spectral norm of $\Delta L$ & SNR  \\
    \midrule
    100  & 84.59 & 0.4543 & {4.80} \\
	200  & 93.73 & 0.4238 & {5.52} \\
	
400 & 96.04 & 0.3982 & 6.24 \\
    
    800  & 99.18 & 0.3725 & 6.96 \\
	\bottomrule
	\end{tabular}
\end{table} 


\begin{table}[H]
\caption{Average accuracy and standard deviation on the largest connected subgraph over 100 trials} \label{tabPlanetoid}
\centering
\begin{tabular} {cllllll}
\toprule
 & \# Labels per class & 0 & 1  & 3 & 5 & 10  \\ 
 \midrule
 \multirow{ 4}{*}{Cora} &
Laplace (LP) & - & 21.8 (14.3) & 37.6 (12.3) & 51.3 (11.9) & 66.9 (6.8)  \\
& Poisson & - & 59.8 (7.9) & 66.2 (5.8) & 72.4 (2.1) & 74.1 (1.8) \\
& PoissonMBO & - & 59.9 (6.4) & 69.1 (3.1) & 72.4 (2.4) & 74.3 (2.1) \\
& CutSSL & - & 67.4 (3.4) & 73.2 (3.1)  & 75.8 (2.1) & \textbf{78.7 (1.1)} \\
& \textbf{SSLC/USLC} & \textbf{64.2 (5.7)} & \textbf{69.2 (4.0)} & \textbf{75.5 (3.4)} & \textbf{76.9 (2.6)} & 77.9 (1.3) \\ 
\midrule
\multirow{ 4}{*}{CiteSeer} &
Laplace (LP) & - & 27.9 (10.4) & 47.6 (8.1) & 56.0 (5.9) & 63.7 (3.5) \\
& Poisson & - &  59.4 (5.4) & 59.4 (5.4) & 62.7 (4.2) & 66.9 (1.8) \\
& PoissonMBO & - & 47.7 (8.0) & 55.7 (3.2)  & 61.0 (1.7) & 63.1 (1.7) \\
& CutSSL & - &  62.4 (4.6) & 63.4 (7.2) & 66.9 (1.4) & 68.1 (1.3) \\
& \textbf{SSLC/USLC} & \textbf{59.8 (5.7)} & \textbf{65.2 (5.3)} & \textbf{67.1 (4.7)} & \textbf{68.6 (2.9)} & \textbf{69.3 (1.6)} \\ 
\midrule
\multirow{ 4}{*}{PubMed} &
Laplace (LP) & - & 34.6 (8.8) & 35.7 (8.2)  & 36.9 (8.1) & 39.6 (9.1)  \\
& Poisson & - &  56.7 (12.8) & 66.5 (6.6) & 68.4 (5.9) & 71.2 (3.4) \\
& PoissonMBO & - & 56.9 (7.3) & 67.9 (3.4) & 69.6 (3.1) & 71.4 (2.5) \\
& CutSSL & - & 63.1 (4.7) & 70.4 (3.1) & 72.8 (2.9) & \textbf{74.1 (1.4)} \\
& \textbf{SSLC/USLC} & \textbf{61.6 (6.2)} & \textbf{67.3 (5.1)} & \textbf{71.7 (3.9)} & \textbf{73.1 (2.9)} & 73.9 (2.2)  \\ 
\bottomrule
\end{tabular}
\end{table}

\begin{figure}[h]
    \centering
    \includegraphics[width=0.4\linewidth]{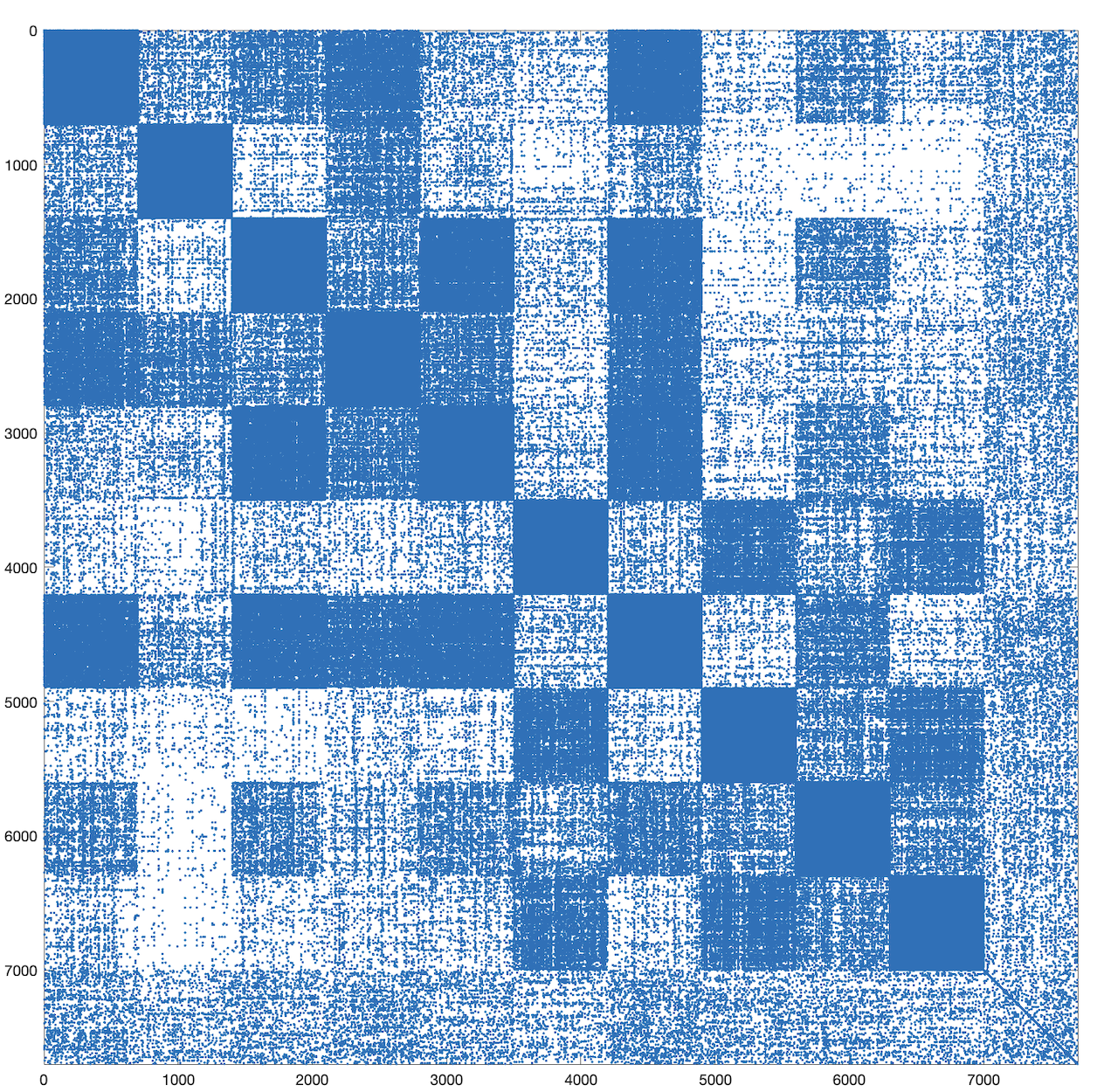}
    \caption{Affinity matrix of FashionMNIST after adding the outlier images (the last block consists of $10\%$ outliers compared to the size of original dataset)}
    \label{figOutlier}
\end{figure}

\section{Notations} \label{secAppendixNotations}

\begin{table}[H]
    \centering
    \caption{Table of Notations} \label{Notation}
    \vspace{3mm}
    \begin{tabular}{cl}
        \toprule
         Symbols & Interpretations \\
         \midrule
         $G$ & general graph of interest \\
         $E$ & set of edges of graph G \\
        $V$ & set of nodes in $G$ (size denoted by $n$) \\
        $C_s$ & each underlying true cluster \\
        $C_s^\#$ & each extracted cluster from algorithm \\
        $\Gamma_s$ & set of Seeds for each cluster \\
        $U$ & removal set from $V$ in Algorithm~\ref{alg1}\\
        $G^{in}$ & subgraph of $G$ on $V$ with edge set $E^{in}$ \\
        $G^{out}$ & subgraph of $G$ on $V$ with edge set $E^{out}$ \\
        $E^{in}$ & subset of $E$ which consists only intra-connection edges \\
        $E^{out}$ & the complement of $E^{in}$ within $E$ \\
        
        $A \text{ } (A^{in})$ & adjacency matrix of $G \text{ } (G^{in})$ \\
         $L \text{ } (L^{in})$ & random walk Laplacian matrix of $G \text{ } (G^{in})$ \\
        $L_{C} \text{ } (L^{in}_C)$ & submatrix of $L \text{ } (L^{in})$ with column indices $C\subset V$ \\
        $\ell_i \text{ } (\ell_i^{in})$ & $i$-th column of $L \text{ } (L^{in})$ \\
        $L^{in}_{\Omega}$ & submatrix of $L^{in}$ with column indices $\Omega\subset V$ \\
        $|M|$  & entrywised absolute value operation on matrix $M$ \\ 
        $\|M\|_2$ & $\|\cdot\|_2$ norm of matrix $M$ \\
        $|\mathbf{v}|$ & entrywised absolute value operation on vector $\mathbf{v}$ \\
        $\|\mathbf{v}\|_2$ & $\|\cdot\|_2$ norm of vector $\mathbf{v}$.
         \\
         $\mathbf{1}_{C}$  & indicator vector on subset $C\subset V$ \\
         $\triangle$ & set symmetric difference \\
         $\mathcal{L}_s(\mathbf{v})$ & $\{i\in[n]: \text{$v_i$ among $s$ largest-in-magnitude entries in $\mathbf{v}$}\}$ \\
       \bottomrule
    \end{tabular}
\end{table}

\end{document}